\documentclass[11pt]{article}

\usepackage[colorlinks,linkcolor=blue,citecolor=blue,urlcolor=blue]{hyperref}

\usepackage{url}
\usepackage{graphicx}
\usepackage[font=scriptsize]{caption}
\usepackage{amsmath}

\usepackage{amsfonts}
\newcommand\numberthis{\addtocounter{equation}{1}\tag{\theequation}}
\usepackage{xcolor}
\usepackage{algorithm,algorithmic}
\usepackage{tabularx} % also loads 'array' package
\usepackage{amsthm}
\usepackage{amsfonts}
\usepackage{amsmath}
\usepackage{amssymb}
\usepackage{mathrsfs}
\usepackage{multirow}
\usepackage{bbm}
\usepackage[round,sort]{natbib}
\usepackage{fullpage}

\newtheorem{theorem}{Theorem}
\newtheorem{definition}{Definition}
\newtheorem{corollary}{Corollary}
\newtheorem{remark}{Remark}

\newtheorem{proposition}{Proposition}
\newtheorem{lemma}{Lemma}
\newcounter{assump}
\newtheorem{assumption}[assump]{Assumption}

\DeclareMathOperator*{\argmin}{arg\,min}

\DeclareMathOperator*{\sol}{SOL}

\DeclareMathOperator*{\vi}{VI}

\newcommand{\st}{\text{ s.t. }}

\newcommand{\beq}{\begin{equation}}

\newcommand{\eeq}{\end{equation}}
\newcommand{\beqs}{\begin{equation*}}
\newcommand{\eeqs}{\end{equation*}}

\renewcommand{\S}{\mathcal{S}}

\newcommand{\K}{\mathcal{K}}

\newcommand{\R}{\mathbb{R}}

\newcommand{\U}{\mathcal{U}}

\newcommand{\M}{\mathcal{M}}

\newcommand{\define}{\stackrel{\text{def.}}{=}}

\newcommand{\E}{\mathbb{E}}

\newcommand{\D}{\mathcal{D}}

\newcommand{\<}{\langle}
\renewcommand{\>}{\rangle}

\renewcommand{\Pr}{\mathbb{P}}

\usepackage{authblk}

\begin{document}

\title{Dynamic Regret Convergence Analysis and an Adaptive Regularization Algorithm for On-Policy Robot Imitation Learning}
	
	\author[1]{Jonathan N. Lee\thanks{jonathan\_lee@berkeley.edu}}
	\author[2]{Michael Laskey}
	\author[1]{Ajay Kumar Tanwani}
	\author[1]{\\Anil Aswani}
	\author[1]{Ken Goldberg}
	\affil[1]{University of California, Berkeley}
	\affil[2]{Toyota Research Institute}

	\renewcommand\Authands{ and }

\maketitle

\begin{abstract}

On-policy imitation learning algorithms such as DAgger evolve a robot control policy by executing it, measuring performance (loss), obtaining corrective feedback from a supervisor, and generating the next policy. As the loss between iterations can vary unpredictably, a fundamental question is under what conditions this process will eventually achieve a converged policy.  If one assumes the underlying trajectory distribution is static (stationary), it is possible to prove convergence for DAgger. However, in more realistic models for robotics, the underlying trajectory distribution is dynamic because it is a function of the policy. Recent results show it is possible to prove convergence of DAgger when a regularity condition on the rate of change of the trajectory distributions is satisfied. In this article, we reframe this result using dynamic regret theory from the field of online optimization and show that dynamic regret can be applied to any on-policy algorithm to analyze its convergence and optimality. These results inspire a new algorithm, Adaptive On-Policy Regularization (\textsc{Aor}), that ensures the  conditions for convergence. We present simulation results with cart-pole balancing and locomotion benchmarks that suggest \textsc{Aor} can significantly decrease dynamic regret and chattering as the robot learns. To our knowledge, this the first application of dynamic regret theory to imitation learning.

\end{abstract}

\maketitle

\section{Introduction}\label{sec:intro}

There has been great interest in learning-based methods to enable robots to accomplish difficult tasks autonomously. One of the foremost methods is learning by demonstration, also known as imitation learning.
In imitation learning, a robot observes states and control labels from a supervisor and estimates a mapping from states to controls. A fundamental problem in imitation learning is covariate shift \citep{bagnell2015invitation}, where the distribution of trajectories experienced by the robot at run time differs from the distributions experienced during training time. For example, consider an autonomous vehicle trained to drive on a road from demonstrations of humans driving safely on the center of the same road. If the vehicle makes slight errors when it is deployed, it may drift towards the sides of the road where it had not previously experienced data from human supervisors, resulting in situations from which it cannot recover.

On-policy Imitation learning algorithms such as DAgger \citep{ross2010reduction}, AggreVaTeD \citep{sun2017deeply}, LOKI \citep{cheng2018fast}, and SIMILE \citep{le2016smooth} have been proposed to mitigate this issue. As opposed to learning only from supervisor demonstrations, these algorithms roll out the robot's current policy at each iteration, allowing it to make errors and observe new states. The supervisor then provides corrective control labels for these new states retroactively. For this reason, these algorithms are often referred to as on-policy imitation learning algorithms because the robot iteratively learns from its current policy \citep{osa2018algorithmic} much like an on-policy reinforcement learning algorithm. This is in contrast to off-policy algorithms where the robot learns from passively observing the supervisor's demonstrations.

On-policy algorithms for imitation learning have been extremely successful in robot learning applications \citep{ross2013learning,pan2017agile,duvallet2013imitation,duan2017one,zhang2017query}. However, the theoretical understanding of these algorithms is limited. For a given algorithm, we might ask the following fundamental questions: Does it return an ``optimal'' solution? Does it provably converge to a policy? If so, what are the properties of this converged policy? The original work of \cite{ross2010reduction} on DAgger sought to bound the worst-case average loss of a policy returned from the algorithm. However, the aforementioned questions remained open. Recently, the work of \cite{cheng2018convergence} proved convergence for the case of DAgger under regularity conditions on the dynamics. In this work, we aim to extend these prior results by analyzing a general formulation of on-policy algorithms in an effort to answer the above questions directly. We examine these algorithms through the lens of modern results in dynamic regret theory from the field of online optimization \citep{zinkevich2003online}.

%Recently, there has been interest in determining when on-policy algorithms are guaranteed to converge to good policies because practitioners often settle on the policy from the final iteration of the algorithm.  
%Cheng and Boots \cite{cheng2018convergence} proved that \textsc{Dagger} converges under the condition that a sensitivity parameter related to the rate of change of the trajectory distributions is small.

% \mlnote{why?, justify this with Cheng's example of using the last policy}

On-policy algorithms can be viewed as derivatives of algorithms from online optimization \citep{ross2010reduction, hazan2016introduction}, a rich theoretical field often used to analyze problems such as portfolio management and network routing analysis where environments change over time \citep{hazan2007adaptive}. The general forumation is as follows: at iteration $n$, a learner makes a decision $\theta_n$ from a set of decisions $\Theta$ and then the environment presents some loss function $f_n :\Theta \mapsto \R$. The learner aims to minimize the incurred losses $f_n(\theta_n)$ over time. In on-policy imitation learning, $\theta_n$ would correspond to the parameters of the robot's policy. The loss $f_n$ would be a supervised learning loss function obtained from rolling out $\theta_n$ and observing corrective labels from the supervisor. At each iteration the supervised learning loss function is different because the distribution of trajectories induced by the robot changes as the policy is updated. The goal is to continually try to find the optimal $\theta_n$ at each iteration based on the sequence of past loss functions. 
A common choice in online optimization to measure the performance of an algorithm  is static regret over $N$ iterations, defined as
\begin{align}\label{eq:static-regret}
R_S(\theta_1, \ldots, \theta_N) \define \sum_{n = 1}^N f_n(\theta_n) - \min_{\theta} \sum_{n = 1}^N f_n(\theta).
\end{align}
In imitation learning, this would mean the robot is compared against the best it could have done on the average of its trajectory distributions seen in the past.

% \mlnote{cut all mention of static regret and start here after the introduction of OCO, your goal is to understand convergence ... so why does static regret matter in the intro}
In this article, we focus on determining when on-policy algorithms can and cannot converge and specifically when the policy is performing optimally on its own distribution. We draw a connection between this very natural objective and an alternative metric known as \textit{dynamic regret}, which has recently gained significant traction in the online optimization community \citep{hall2015online,mokhtari2016online,yang2016tracking}. As opposed to the well known static regret, dynamic regret measures performance of a policy at each instantaneous iteration:
\begin{align}\label{eq:dynamic-regret}
R_D(\theta_1, \ldots, \theta_N) \define \sum_{n = 1}^N f_n(\theta_n) - \sum_{n = 1}^N \min_{\theta} f_n(\theta).
\end{align}

The difference between static and dynamic regret is that, for dynamic regret, the minimum resides inside the summation, meaning that the regret is an instantaneous difference at each iteration. Proving that static regret is low implies that the policies are at least as good as a single fixed policy that does well on the average of the distributions seen during training. This was precisely the original result of \cite{ross2010reduction}. In constrast, proving that dynamic regret is low implies that the policies are optimal on their own distribution. Consequently, dynamic regret can be more relevant as a theoretical metric in robotics where the trajectory distributions are changing, and understanding the dynamic regret of an algorithm positions us to answer the aforementioned questions about on-policy imitation learning.

As an example of the difference in practice, consider again the autonomous driving scenario using an on-policy algorithm where the car will encounter both challenging safety-critical situations, such as avoiding collisions, and relatively easy situations, such as lane-following. Low static regret implies that car the will do well in many of these situations on average. However, this could mean the policy is good at lane-following but highly suboptimal at avoiding collisions even if it has to encounter both situations at run time. However, low dynamic regret offers a much safer guarantee by ensuring that the car will be optimal (among its policy class) with respect to the situations it actually encounters at run time.

However, showing an algorithm achieves low dynamic regret is inherently harder than achieving low static regret because $R_S \leq R_D$. Furthermore, it is well known that it is not possible to guarantee low dynamic regret in general due to the possibility of adversarial loss functions \citep{yang2016tracking}. The key insight in imitation learning, however, is that the loss functions at each iteration represent the trajectory distributions as a function of the  policy parameters. Therefore, we can leverage information known about how the trajectory distribution changes in response to changing policy parameters to obtain interpretable dynamic regret rates. In particular, we rely on the continuity of the trajectories as a function of the policy \citep{cheng2018convergence}. 

This paper makes four contributions:
\begin{enumerate}
	\item Introduces a novel dynamic regret analysis to evaluate the convergence of on-policy imitation learning algorithms.
	\item Presents average dynamic regret rates and conditions for convergence for DAgger, Imitation Gradient, and Multiple Imitation Gradient.
	\item Introduces Adaptive On-Policy Regularization (\textsc{Aor}), a novel algorithm that adaptively regularizes on-policy algorithms to improve dynamic regret and induce convergence.
	\item Presents empirical evidence of non-convergent on-policy algorithms and shows that \textsc{Aor} can ensure convergence in a cart-pole balancing task and locomotion tasks.
\end{enumerate}

This article is a significantly revised and extended version of our conference publication at the Workshop on the Algorithmic Foundations of Robotics \citep{lee2018dynamic,lee2018stability}. In particular, this article (1) presents new theoretical results, greatly extending the formalization of dynamic regret as a metric in imitation learning; (2) provides detailed examples and analysis of well known systems that satisfy the continuity condition required in the theory; (3) explores connections with our subsequent work in Continuous Online Learning \citep{cheng2019online} and the variational inequality problem;
(4) presents new experimental results.

\section{Related Work}\label{sec:related}
The challenge of covariate shift in imitation learning by supervised learning is the subject of significant research in robotics. It is especially prevalent when the robot's policy cannot fully represent the supervisor \citep{laskey2017comparing}. In robotics, many algorithms have been proposed to mitigate covariate shift for imitation learning. \cite{ross2010reduction} introduced DAgger, an on-policy algorithm that allows the robot to make mistakes and then observe corrective labels from the supervisor in new states that might not be seen from ideal supervisor demonstrations alone. 

Gradient-based on-policy methods for imitation learning have gained interest due to their similarity to policy gradient algorithms and their computational efficiency. These are also known as Imitation Gradient methods. AggreVaTeD \citep{sun2017deeply} was proposed for fast policy updates designed for deep neural network policies. LOKI \citep{cheng2018fast} uses a mirror descent algorithm on an imitation learning loss to bootstrap reinforcement learning.

\cite{ross2010reduction} first introduced a static regret-style analysis for DAgger, showing that with strongly convex losses, running DAgger results in low static regret in all cases. This means that a DAgger policy is on average at least as good as one policy that does well on the average of trajectory distributions seen during training. However, the performance on the average of trajectory distributions is not always informative, as shown by \cite{laskey2017comparing}, because the average may contain irrelevant distributions as a result of rolling out suboptimal policies. \cite{cheng2018convergence} recently expanded the DAgger analysis proving that, despite guaranteed convergence in static regret, the algorithm may not always converge to a low loss policy. Furthermore, they identified regularity conditions sufficient for convergence to local optima. This work extends the results of \cite{cheng2018convergence} by drawing a connection with dynamic regret theory to identify conditions for convergence for DAgger and other on-policy algorithms and as a basis for a new algorithm. 

In \cite{cheng2019online} (subsequent to \cite{lee2018dynamic}), we presented a notion of ``continuous online learning" where we studied the dynamic regret of online optimization problems with smoothly changing losses. In that work, we generalized the structure observed in the imitation learning problem to prove low dynamic regret more generally. We discuss connections to the continuous online learning problem in Section \ref{sec:col}.
%\red\sout{Portions of this paper were presented but not published at a RSS workshop \cite{lee2018stability}}}. 
%To the best of our knowledge, this is the first application of dynamic regret theory to imitation learning.

\section{Preliminaries}\label{prelims}

In this section, we review mathematical background and notation. We then introduce the problem statement for imitation learning by supervised learning.

\subsection{Mathematical Background}
Let $\Theta \subset \R^d$ be a convex and compact set with $l_2$-norm $\|\cdot \|$.
A differentiable function $h:\Theta \mapsto \R$ is said to be convex if $\forall \theta, \theta' \in \Theta$, \begin{align*}h(\theta) \geq h(\theta') + \<\nabla h(\theta'), \theta - \theta'\>.\end{align*}  It is strictly convex if the inequality is strict for all $\theta \neq \theta'$. It is $\alpha$-strongly convex if there exists $\alpha > 0$ such that \begin{align*}h(\theta) \geq h(\theta') + \< \nabla h(\theta') , \theta - \theta'\> + \frac{\alpha}{2}\| \theta - \theta'\|^2.\end{align*}
 Strongly convex functions are strictly convex, and strictly convex functions are convex. The function $h$ is said to be $\gamma$-smooth if its gradients are $\gamma$-Lipschitz continuous:
\begin{align*}
\| \nabla h(\theta) - \nabla h(\theta') \| \leq \gamma \| \theta - \theta'\|.
\end{align*}
An equivalent definition of $\gamma$-smoothness is that for all $\theta, \theta' \in \Theta$, 
\begin{align*}
h(\theta) \leq h(\theta') + \< \nabla h(\theta'), \theta - \theta'\> + \frac{\gamma}{2}\|\theta - \theta'\|^2,
\end{align*}
which parallels the definition of strong convexity. The Bregman divergence $\D_h : \Theta \times \Theta \mapsto \R_{\geq 0}$  with respect to a strictly convex function $h$ is defined as 
\begin{align*}
\D_h(\theta, \theta') = h(\theta) - h(\theta') - \< \nabla h(\theta'), \theta - \theta'\>.
\end{align*}

\subsection{Imitation Learning}
\paragraph{Setup.} We consider systems represented by a Markov decision process (MDP)  $ \M = \< \mathcal S, \U, p, T\>$, where $\mathcal S$ is a bounded set of states, $\U$ is a bounded set of allowable controls, $p: \mathcal S \times \U \mapsto \Pr(\mathcal S)$ is transition function mapping states and actions to a density over states, and $T$ is the time horizon. We do not assume access to a reward or cost function intrinsic to the MDP. In this setting $\S$ and $\U$ can be discrete or continuous.

A policy $\pi : \S \mapsto \Pr(\U)$ from the set of allowable policies $\Pi$ is a mapping from states to a density over the set of controls. When a policy is deployed on the MDP, such as by rolling out a robot in the real world, it generates a trajectory $\tau$, which is a tuple of states and actions: $\tau = \left(s_1, u_1, \ldots, s_{T - 1}, u_{T - 1}, s_T  \right)$. Given an initial state distribution $\mu$ over $\S$, the probability density of a trajectory under policy $\pi$ is given by  
%In this article, we will consider only learning time-invariant policies.
%Let $p_0$ denote an initial state distribution over $\S$. 
%Let $s_t \in \mathcal S$ and $u_t \in \mathcal U$ be the state and control in a Markov decision process. For a given state and control, the transition distribution is represented by $p(\cdot | s_t, u_t) : \mathcal S \mapsto \R_{\geq 0}$ with initial state distributoin $p_0 : \mathcal S \mapsto \R_{\geq 0}$. The probability (or density) of a trajectory $\tau$ of length $T$ under a (possibly stochastic) policy $\pi \in \Pi: \mathcal S \mapsto \Pr (\U)$ is given by
\begin{align*}
p(\tau;\pi) = \mu(s_1) \prod_{t = 1}^{T - 1} \pi (u_t | s_t) p(s_{t+1} | s_t, u_t),
\end{align*}
When $\pi$ is deterministic, it is taken to be the Dirac (or Kronecker in the discrete case) delta function. In this paper, we consider learning parametric policies, i.e., there is a convex and compact subset of parameters $\Theta \subset \R^d$ equipped with the $l_2$-norm $\| \cdot \|$ that parameterizes policies $\pi_\theta$ for $\theta \in \Theta$. Its diameter is given by $D\define \max_{\theta_1, \theta_2 \in \Theta} \|\theta_1 - \theta_2\|$. 

\paragraph{Objective.} Rather than assuming access to a cost or reward function, we assume that the robot has access to a supervisor policy $\pi^*$ that is capable of providing control labels to states. Our goal is to find a parameterized policy $\pi_\theta$ that minimizes some loss with respect the supervisor $\pi^*$. The learned policy in general is not able to match $\pi^*$. The loss of a policy $\pi_\theta$ along a trajectory $\tau$ is a non-negative function such that
\begin{align*}
J_\tau(\pi_\theta) \define  %\E_{\pi_\theta} 
\frac{1}{T - 1} \sum_{t = 1}^{T - 1} \ell_t(s_t,\pi_\theta),
%J_\tau(\pi) \define  \E_{\pi,\pi^*} \frac{1}{T} \sum_{t = 1}^{T - 1} \ell(\hat u_t, u_t^*),
\end{align*}
where $\ell_t: \S \times \Theta \mapsto \mathbb R_{\geq 0}$ is a differentiable (in $\theta$) per-time step surrogate loss function. %, $\hat u_t \sim \pi(\cdot | s_t)$ and $u_t^* \sim \pi^*(\cdot | s_t)$. 
%For parameterized policies, we will abuase notation and write $J_\tau(\theta) \define J_\tau(\pi_\theta)$.
%The expectation is taken over trajectories $\tau \sim p(\tau;\pi_\theta)$.
The general optimization problem of imitation learning can be written as
\begin{align*}
\min_{\theta } \quad  \E_{p (\tau; \pi_\theta)} J_\tau( \pi_\theta) \quad \st \quad \theta \in \Theta
\end{align*}
where
\begin{align*}
\E_{p (\tau; \pi_\theta)} J_\tau( \pi_\theta) =  \int p(\tau; \pi_\theta) J_\tau(\pi_\theta) d\tau
\end{align*}

The expectation is taken over the distribution of trajectories that $\pi_\theta$ induces. We can interpret this objective in the following way: trajectories are sampled from the distribution induced by $\pi_\theta$ and then the performance of $\pi_\theta$ is evaluated along those trajectories. We wish to find the parameter $\theta$ in $\Theta$ that gives rise to trajectories that ensure that 
$J_\tau(\theta)$ is low.
%$\theta$ has low loss on those trajectories. 
This problem reflects the goal of having the policy do well on its own induced distribution. However, this is challenging and cannot be solved efficiently \citep{sun2017deeply} since the distribution of trajectories is unknown. It also cannot be exactly solved with regular supervised learning by sampling supervisor trajectories because the sampling distribution is a function of the policy \citep{bagnell2015invitation}.

% \mlnote{not quite true, DAgger is not a fix distribution}
%Because this problem cannot be solved directly, 
\paragraph{Relaxations.} Existing algorithms relax the problem by fixing the trajectory distribution and then optimizing over the evaluation parameter. This decouples the sampling from the supervised learning problem. For example, in standard behavior cloning, one sets the trajectory distribution to the supervisor's trajectory distribution and finds the policy that minimizes loss on that distribution. Formally, we consider the average loss of a parameter $\theta \in \Theta$ over the distribution of trajectories generated by a possibly different policy parameter $\theta' \in \Theta$:
\begin{align}\label{imitation-obj}
f_{\theta'}(\theta)\define \mathbb E_{p (\tau; \pi_{\theta'})}  J_\tau( \pi_{\theta}).
\end{align}
Here, $\theta'$ controls the trajectory distribution and $\theta$ controls the predictions used to the compute the loss on that distribution. For this reason, we refer to the $\theta'$ as the distribution-generating parameter and $\theta$ as the evaluation parameter. We can view $f_{\theta'}(\theta)$ as a bifunction \citep{cheng2019online}.

% \mlnote{the next paragraph feels very dense and too soon, I would say cut it because we give the algorithmic block for each algorithm later}
Optimization problems in this article will be of the form $\min_{\theta \in \Theta} f_{\theta'}( \theta)$ for some fixed and known $\theta' \in \Theta$. Because the trajectory distribution no longer depends on the variable $\theta$, the supervised learning problem can now be feasibly solved by sampling from $p(\tau; \pi_{\theta'})$ which corresponds to rolling out trajectories under the fixed policy $\pi_{\theta'}$. Specifically in this article, we will consider on-policy algorithms which operate over $N \in \mathbb N$ iterations. At any iteration $n$ for $1 \leq n \leq N$, the policy parameter $\theta_n$ is rolled out as the distribution-generating parameter and the loss $f_{\theta_n}(\theta) = \mathbb E_{p (\tau; \pi_{\theta_n})} J_\tau(\pi_\theta)$ is observed, where $\theta$ is the free variable. For convenience, we write $f_n(\theta)\define f_{\theta_n}(\theta)$. These loss functions form the sequence of losses used in the regret metrics $R_S$ and $R_D$.

%\subsection{Online Optimization}

%Online optimization is a field of its own within the broader umbrella of theoretical optimization. For this reason, we outline the basic goals and problem settings.

%\paragraph{Objective.} The goal in online optimization is to design algorithms that generate sequences of parameters $\{ \theta_n\}_{n = 1}^N$ such that the cumulative loss $\sum_{n = 1}^N f_n(\theta_n)$ (is small) in some respect. An algorithm is regarded as performing well if its regret (e.g. $R_S$ or $R_D$) is \textit{sublinear} in $N$, written as $o(N)$ \citep{hazan2016introduction}. That is, it performs at least as well as a static or dynamic comparator on average. 

\subsection{Assumptions}
Finally, we briefly describe the main assumptions of this article. The assumptions are stated formally in Section \ref{guarantees}. As in prior work in both imitation learning and online optimization \citep{ross2010reduction,hazan2016introduction}, we assume strong convexity and smoothness of the loss function in the evaluation parameter $J_\tau(\pi_\theta)$. Intuitively, strong convexity ensures the loss is curved at least quadratically while smoothness guarantees it is not too curved. In practice these assumptions are satisfied when $J_\tau(\pi_\theta)$ is a ridge regression loss, for example.

 As in the work of \cite{cheng2018convergence}, we also assume a regularity condition on the $f_n$ sequence, which bounds the sensitivity of the trajectory distribution in response to changes in the distribution-generating parameter. That is, for a fixed $\theta$, we assume that $\nabla f_{\theta'}(\theta)$ is Lipschitz continuous in $\theta'$. 
 This condition 
% is arguably strict from an online optimization perspective but it 
 captures the basic continuity and structure of the imitation learning problem. We discuss its motivation in Section~\ref{sec:dyn} and examples in Section~\ref{sec:examples}.

In practice, many works have observed good performance of both off-policy and on-policy algorithms using non-convex losses such as neural networks even though these algorithms were often proposed for convex losses. \citep{pomerleau1989alvinn,laskey2017comparing,zhang2018deep,hussein2018deep,pan2017agile}.

\section{On-Policy Algorithms}
We now review three on-policy algorithms that will be the focus of the main theoretical and empirical results of the article.

\begin{figure}
	\begin{algorithm}[H]
		\caption{DAgger \citep{ross2010reduction}}
		\label{alg:ftl}
		\begin{algorithmic}
			\STATE {\bfseries Input:}  Initial policy parameter $\theta_1$,
			\STATE ~~~~~~~~~~ Max iterations $N$.
			\FOR{$n=1$ {\bfseries to} $N-1$}
			\STATE Roll out $\theta_n$ and collect $\tau_n$.
			\STATE Form loss $f_n(\theta) = f_{\theta_n}(\theta)$ from feedback on $\tau_n$.
			\STATE $\theta_{n + 1} \gets \argmin_{\theta \in \Theta} \sum_{m=1}^n f_m(\theta)$.
			\ENDFOR
		\end{algorithmic}
	\end{algorithm}
	\begin{algorithm}[H]
		\caption{Imitation Gradient \citep{cheng2018fast,sun2017deeply}}
		\label{alg:ogd}
		\begin{algorithmic}
			\STATE {\bfseries Input:}  Initial policy parameter $\theta_1$,
			\STATE ~~~~~~~~~~ Max iterations $N$,
			\STATE ~~~~~~~~~~ Stepsize $\eta$.
			\FOR{$n=1$ {\bfseries to} $N-1$}
			\STATE Roll out $\theta_n$ and collect $\tau_n$.
			\STATE Form loss $f_n(\theta) = f_{\theta_n}(\theta)$ from $\tau_n$.
			\STATE  $\theta_{n + 1} \gets P_\Theta \left(\theta_n - \eta \nabla f_n(\theta_n) \right)$.
			\ENDFOR
		\end{algorithmic}
		\end{algorithm}
		\begin{algorithm}[H]
		\caption{Multiple Imitation Gradient \citep{zhang2017improved}}
		\label{alg:omgd}
		\begin{algorithmic}
			\STATE {\bfseries Input:}  Initial policy parameter $\theta_1$,
			\STATE ~~~~~~~~~~ Max iterations $N$,
			\STATE ~~~~~~~~~~ Updates per iteration $K$,
			\STATE ~~~~~~~~~~ Stepsize $\eta$.
			\FOR{$n=1$ {\bfseries to} $N-1$}
			\STATE Roll out $\theta_n$ and collect $\tau_n$.
			\STATE Form loss $f_n(\theta) = f_{\theta_n}(\theta)$ from $\tau_n$.
			\STATE $\theta_n^1 \gets \theta_n$
			\FOR{$k=1$ {\bfseries to} $K$}
			\STATE  $\theta_n^{k + 1} \gets P_\Theta \left(\theta_n^k - \eta \nabla f_n(\theta_n^k) \right)$.
			\ENDFOR
			\STATE $\theta_{n + 1} \gets \theta_n^{K+1}$.
			\ENDFOR
		\end{algorithmic}
	\end{algorithm}
%zhang2017improved
\caption*{Top: DAgger minimizes over all observed loss functions which are represented by a supervised learning loss over all observed data. Middle:
	Imitation Gradient computes a gradient on data collected from only the most recent rollout and takes a single gradient step. $P_\Theta$ is a projection operation, projecting the resulting paramter vector back onto $\Theta$ in the event the gradient updates leaves it outside. Bottom: Multiple Imitation Gradient is a generalization of Imitation Gradient where $K \geq 1$ gradient steps are take with each rollout.
}
\end{figure}

\subsection{DAgger}
DAgger is a derivative of the follow-the-leader algorithm from Online Optimization. For detailed discussion of implementation, we refer the reader to \citep{ross2010reduction,hazan2016introduction}. DAgger proceeds by rolling out the current policy and observing a loss based on the induced trajectory distribution. The next policy parameter is computed by aggregating all observed losses and minimizing over them. An example of this for the $l_2$-regularized linear regression problem would be $\min_\theta \sum_{m = 1}^n f_m(\theta)$, where
\begin{align}\label{ridge-regression}
\sum_{m = 1}^n f_m(\theta) = \sum_{m = 1}^n \mathbb E \|S_m \theta - U_m\|^2 + \frac{\alpha_0}{2} \|\theta\|^2.
\end{align}
Here, $S_m$ is the matrix of state vectors observed from rolling out $\pi_m$ and $U_m$ is the matrix of control labels from the supervisor at the $m$th iteration. The expectation is taken with respect to the randomness in $S_m$ and $U_m$. In the original DAgger algorithm, a user-defined stochastic mixing term was included. We omit this term because setting it to zero only improves the original bound by \cite{ross2010reduction}.

\subsection{Imitation Gradient}
Recently there has be interest in ``Imitation Gradient" algorithms. Algorithms such as AggreVaTeD and LOKI fall in this family. The online gradient descent algorithm from online optimization underlies such algorithms and their variants. The aforementioned methods have explored more general gradient-based algorithms such as those inspired by the natural gradient and mirror descent. The analysis in this article will focus on the basic online gradient descent algorithm.
Online gradient descent proceeds by observing $f_n$ at each iteration and taking a weighted gradient step: $\theta_n - \eta \nabla f_n(\theta_n)$. In the event that the resulting parameter lies outside of $\Theta$, it is projected back on the space with the projection $P_\Theta(\theta) = \argmin_{\theta' \in \Theta} \|\theta' - \theta\|$.
To address the mirror descent case, we would instead define the updates as
\begin{align*}
\theta_{n+1} = \argmin_{\theta \in \Theta} \quad \eta \< \nabla f_n(\theta_n), \theta\> + \D_h(\theta, \theta_n),
\end{align*}
where again $\D_h$ is the Bregman divergence with respect a strictly convex funciton $h$. Taking $h(\theta) = \frac{1}{2}\|\theta\|^2$ recovers the basic imitation gradient algorithm. 

Again, using the $l_2$-regularized regression example, the update would take the form
\begin{align*}
\theta_{n+1} = P_\Theta\left( \theta_n  - 2 \eta \left(  \E \left[S_n^\top (S_n -  U_n ) \right]  + \frac{\alpha_0}{2} \theta_n \right) \right).
\end{align*}
In online optimization (and general convex optimization), the step size is sometimes taken to be variable in order to ensure convergence under different assumptions.

\subsection{ Multiple Imitation Gradient}
We also consider a related algorithm termed Multiple Imitation Gradient, based on the online multiple gradient descent algorithm. This algorithm is identical to the Imitation Gradient, but at each iteration it updates the policy parameters $K$ times, recomputing the gradient each time. Imitation Gradient is a special case of Multiple Imitation Gradient where $K = 1$. %The algorithms are shown together in Algorithm \ref{alg:ogd}.
It was shown by \cite{zhang2017improved} that, for dynamic regret, multiple imitation gradient is asymptotically at least as good as imitation gradient and in some cases it achieves improved rates.

\section{Towards Dynamic Regret in Imitation Learning}\label{sec:dyn}

The imitation learning algorithms presented in the preceding section are derivatives of standard online optimization algorithms. At first glance, the underlying objectives in both fields also seem aligned: find a sequence of policies (parameters) $\{\theta_n\}_{n=1}^N$ such that the cumulative loss $\sum_{n = 1}^N f_n(\theta_n)$ is low in some regard. It is only natural then that they are analyzed in a familiar online optimization framework. An algorithm is regarded as performing well if its regret (e.g. $R_S$ or $R_D$) is \textit{sublinear} in $N$, written as $o(N)$ \citep{hazan2016introduction}. That is, it performs at least as well as a static or dynamic comparator on average. We note that some authors use the term ``no-regret" \citep{ross2010reduction}, which means the same as sublinear regret. No-regret refers to the \textit{average} regret $\frac{1}{N}R_S$ vanishing, which occurs if and only if it is sublinear.

However, we must be careful when applying off-the-shelf results. In this section, we discuss the limitations of the standard static regret analysis from the perpsective of imitation learning and the argue that dynamic regret provides a more informative measure of performance.

%We are interested in showing that the policies generated by these algorithms perform well on the loss on their own induced trajectory distributions and, furthermore, that they converge. This means that we would like the difference $f_n(\theta_n) - \min_{\theta \in \Theta} f_n(\theta)$ to be as small as possible for every iteration $n$. This difference represents the instantaneous regret the algorithm has for playing $\theta_n$ instead of $\theta^*_n$ at iteration $n$ where $\theta^*_n := \argmin_{\theta \in \Theta} f_n(\theta)$. Summing over $n$ from $1$ to $N$, we have the definition of dynamic regret given in Equation (\ref{dynamic_regret}).

\subsection{Limitations of Static Regret}

The more well known static regret, reproduced here as
\begin{align*}
R_S = \sum_{n = 1}^N f_n(\theta_n) - \min_{\theta} \sum_{n = 1}^N f_n(\theta),
\end{align*}
has a rich history of analysis in online optimization. Indeed, this was the first regret definition introduced by \cite{zinkevich2003online}. Algorithms have been proposed and analyzed under varying notions of convexity and smoothness to yield sublinear static regret rates \citep{shalev2009mind,hazan2007logarithmic}. Furthermore, the metric has also been used for analysis of stochastic descent methods in convex optimization \citep{duchi2011adaptive}.

Sublinear static regret implies that the parameters played are at least as good as a single fixed parameter, ensuring that
\begin{align*}
f_N(\theta_N) - \min_{\theta \in \Theta} \frac{1}{N} \sum_{n = 1}^N f_n(\theta)
\end{align*}
is vanishing (or even negative) on average. Achieveing sublinear static regret is important for problems that are online versions of batch learning problems because an algorithm that achieves sublinear static regret is essentially as good as learning if all data was given at once. This validates the use of online learning algorithms for standard statistical learning tasks \citep{rakhlin2014statistical} and mirrors the types of problems one encounters in classical statistics \citep{bubeck2011introduction}. A canonical example is an email spam filter where a classifier is learned over time as samples of spam and true emails are revealed from a static distribution.

In the context of imitation learning, the static regret metric compares an algorithm's sequence of policies to the minimizer over the average of losses on all trajectory distributions seen during training. 
Efficient algorithms exist to achieve static regret for convex and strongly convex losses regardless of any external conditions such as the robot or system dynamics or distributions induced \citep{ross2010reduction}. However, the ability to have low regret compared to the average of the past has little bearing at run time.
 The most we can say about an on-policy algorithm with sublinear dynamic regret is that there exists a policy parameter $\theta$ generated by the algorithm such that,
\begin{align*}
f_\theta(\theta) \leq  \frac{o(N)}{N} + \min_{\theta' \in \Theta}\frac{1}{N}\sum_{n = 1}^N f_n(\theta'),
\end{align*}
where the term $o(N)/N$ is vanishing in $N$. A common custom in static regret analyses for imitation learning is to define $\epsilon_N = \min_{\theta' \in \Theta}\frac{1}{N}\sum_{n = 1}^N f_n(\theta')$ and assume that $\epsilon_N$ is small if the policy class is sufficiently expressive \citep{ke2019imitation}. In practice, this assumption can actually be very strict as pointed out by \cite{laskey2017comparing} and \cite{cheng2018convergence}. It suggests that a policy exists that is simultaneously good on \textit{any} arbitrary collection of $N$ trajectory distributions. That is, for static regret to guarantee a low loss policy, we would need to assume there is nearly zero approximation error (bias) \textit{everywhere} in the state space. Even if such a policy did exist, it was argued by \cite{laskey2018and} that a standard off-policy behavior cloning approach can be more efficient in practice due to increased variance.

%Furthermore, proving that a policy has sublinear regret compared to the average of previously seen trajectory distributions does not give indication of whether the policy actually performs well at run time. 
Thus, the gap $f_\theta(\theta) - \min_{\theta' \in \Theta} f_\theta(\theta')$ could still be very large. The intuitive reason is that some distributions observed during training can be irrelevant to the task due to bad initialization or extremely sensitive dynamics, but they are still included in the average for static regret. 
As a result, static regret can be sublinear regardless of actual policy performance at run time or whether the algorithm leads to convergent policies. Prior work has shown in experiments and theoretical examples that on-policy algorithms can indeed fail on ``hard" problems \citep{cheng2018convergence,laskey2017comparing}, but this is not obvious in static regret theory. Thus, it appears that static regret can be overly optimistic and ambiguous as a theoretical metric for imitation learning unless one assumes that $\epsilon_N$ is arbitrarily small. In order to amend these deficiencies, we turn to dynamic regret.

\subsection{Dynamic Regret}

In online optimization literature, dynamic regret has become increasingly popular to analyze online learning problems where the objective is constantly shifting \citep{hall2015online,hazan2007adaptive,mokhtari2016online,yang2016tracking,zhang2017improved,zinkevich2003online,jadbabaie2015online}. 
In constrast to static regret, sublinear dynamic regret implies that
\begin{align*}
f_N(\theta_N) - \min_{\theta \in \Theta} f_N(\theta)
\end{align*}
is vanishing on average with $N$. Dynamic regret targets the problem statement where it is important achieve low loss on each function observed instantaneously such as portfolio optimization \citep{hazan2007adaptive}. In this setting, one aims to achieve high returns (low loss) at each point in time as the market changes, rather than perform well compared to a fixed benchmark on average. If the market is shifting over time, i.e. prices are changing, we are interested in playing the best strategy for the given state of the market. A related metric, known as adaptive regret, has also been studied. It observes a window of regret rather than instantaneous regret \citep{hazan2007adaptive,adamskiy2012closer}.

To achieve sublinear dynamic regret is difficult, in fact, impossible without further assumptions or knowledge of the underlying environment. For example, suppose the environment plays
\begin{align*}
f_n(\theta) = (\theta - X_n)^2 \quad \text{where} \quad X_n \sim \text{Unif}[-10, 10].
\end{align*} There is no hope that any learner could overcome this volatility to achieve sublinear dynamic regret by simply using past information from observed losses \citep{yang2016tracking,cheng2019online}.

Fundamentally, this difficulty of achieving sublinear dynamic regret is dependent on the change in the loss functions over iterations, often expressed in terms of quantities called variations in the literature. If the loss functions change in an unpredictable or adversarial manner, we can expect large variation terms leading to large regret and suboptimal policies. This is the reason that sublinear dynamic regret bounds cannot be obtained in general using only the assumptions commonly used for static regret \citep{yang2016tracking}. In this article we consider the commonly used path variation and a squared variant of it introduced by \cite{zhang2017improved}.
\begin{definition}[Path Variation and Squared Path Variation]\label{variation}
	For a sequence of optimal parameters from $m$ to $n$ given by  $\theta^*_{m:n}:= \{\theta_i^*\}_{i = m}^n$, the path variation is defined as
	\begin{align*}
	V(\theta^*_{m:n}) \define \sum_{i = m}^{n - 1} \| \theta^*_i - \theta_{i + 1}^*\|,
	%	\quad \text{and} \quad S(\theta^*_{m:n}) := \sum_{i = m}^{n - 1} \| \theta^*_i - \theta_{i + 1}^*\|^2.
	\end{align*}
	and the squared path variation is defined as
	\begin{align*}
	S(\theta^*_{m:n}) \define  \sum_{i = m}^{n - 1} \| \theta^*_i - \theta_{i + 1}^*\|^2.
	\end{align*}
\end{definition}
These measures reflect the variation in the loss functions by computing the distances or the square distances between consecutive optimal parameters. Thus the rate in $N$ of a measure of variation generally influences the rate of the resulting dynamic regret bound and so we can often think of the variation as describing the difficulty of the problem.

Many algorithms have been proposed and analyzed in this regret framework in terms of variation measures of the loss functions. For example, \cite{zinkevich2003online} proved a dynamic regret rate of $O(\sqrt{N}(1 + V(\theta^*_{1:N})))$ for online gradient descent with convex losses. Here, the regret rate is dependent on the rate of the path variation, which might also be a function of $N$. Therein lies the difficulty of achieving low dynamic regret: no matter the algorithm, rates depend on the variation, which can be large for arbitrary sequences of loss functions. 
%Indeed, it was shown by \cite{yang2016tracking} that sublinear dynamic regret is impossible when there is no constraint on the path variation $V(\theta_{m:n}^*)$. 

In the context of imitation learning, dynamic regret compares the $n$th policy to the instantaneous best policy on the $n$th distribution, which means it is a stricter metric than static regret. The advantage of the dynamic regret metric is that a policy's performance at any iteration is always evaluated with respect to the most relevant trajectory distribution: the current one. Proving dynamic regret is low implies that each policy on average is as good as the instantaneous best on its own distribution. Furthermore, we can examine convergence properties and solution characteristics of an algorithm by proving that dynamic regret is low. We formalize this statement in the next section.

% These measures reflect the variation in the loss functions by computing the distances or the square distances between consecutive optimal parameters. Thus the complexity in $N$ of a measure of variation generally influences the complexity of the resulting dynamic regret bound and so we can often think of the variation as describing the difficulty of the problem.
% For strongly convex loss functions considered in this paper, each optimum is unique. It is important to note that despite its name, the squared path variation is not actually the square of the path variation but rather the sum of squares of the lengths comprising the path variation.
%In recent years, the Online Optimization community has seen a surge of interest in dynamic regret problems. 

%Because these Online Optimization algorithms often underlie many on-policy imitation learning algorithms,
We emphasize that the difficulty of achieving sublinear dynamic regret is not a drawback of the dynamic regret metric but rather an important advantage because it reveals whether an imitation learning problem is likely to be solvable by an on-policy algorithm. Achieving low regret implies the algorithm might be successful and find a (near) optimal solution. However, in cases where high dynamic regret is unavoidable, we would be inclined to seek other options.

\subsection{Continuity in the Distribution-Generating Parameter}

Dynamic regret analyses offer a promising framework for theoretical analysis of on-policy imitation learning algorithms but only as long as the variation can be characterized. In imitation learning, there is considerable structure available that is not possible in general-purpose dynamic regret analyses. This structure prevents pathological cases such as the ones discussed previously. We discuss this notion at a high level to provide intuition.

Recall the imitation learning objective function:
\begin{align*}
f_{\theta'} (\theta) = \E_{\tau \sim p(\tau ; \theta')}  J_\tau(\theta).
\end{align*}
Here we abuse notation slightly, writing $p(\tau ; \theta) = p(\tau ; \pi_\theta)$ and $J_{\tau}(\theta) = J_{\tau}(\pi_\theta)$. We first consider an impractical but illustrative example where an on-policy algorithm does not even update the robot's policy. That is,
\begin{align*}
\theta_1 = \theta_2 = \ldots = \theta_N.
\end{align*}
This is a poor strategy, but interestingly the loss functions are identical at each round because the distribution-generating parameter is always the same. This is due to the bifunction form of the objective:
\begin{align*}
f_{\theta_1}(\theta) = \ldots = f_{\theta_N}(\theta)  \quad \forall \theta \in \Theta.
\end{align*}
Thus, the variation is always zero. Still the regret could be large if $ \theta_1$ is not already optimal on its own distribution, but this gives us hope that the variation can be characterized for certain algorithms exploiting this structure.

Now suppose in a simplified example that we perturb $\theta_1$ by a small amount to generate $\theta_2$. We expect the robot to apply similar but different controls than before, but in the worst case these different controls could lead to entirely new trajectories if the dynamics are highly irregular. That is $p(\tau; \theta_1)$ could be very different from $p(\tau; \theta_2)$. The cost $J_\tau(\theta)$ would then be evaluated on vastly different trajectories. In this case, due to the irregularity of the dynamics, the costs at each round of an imitation learning algorithm could be just as difficult as the adversarial case in general-purpose dynamic regret analyses, thus leading to unavoidably high dynamic regret. 

Alternatively, the system may be sufficiently well-behaved that similar controls give rise to similar trajectories. This notion resembles continuity of the trajectory distribution in the distribution-generating parameter: to ensure $\E_{p(\tau; \theta_1)} J(\theta)$ is necessarily close to $\E_{p(\tau; \theta_2)}J_\theta(\theta)$ we should choose $\theta_2$ to be sufficiently close to $\theta_1$. This intuition suggests we should restrict our attention to well-behaved systems that exhibit this continuity to avoid the pathological scenarios discussed previously. Indeed, these well-behaved systems are exactly what one might expect to see in many robotics applications. Such a condition is not unlike differentiability assumptions in reinforcement learning literature \citep{schulman2015trust} or continuity assumptions on dynamics studied in nonlinear systems \citep{khalil2002nonlinear}.  In Section~\ref{sec:examples}, we explore precedents and examples in detail.

Formally, we describe this notion via Lipschitz continuity. We consider systems that satisfy the following condition for some $\beta \geq 0$:
\begin{align}\label{eq:regularity}
\| \nabla f_{\theta_1} (\theta) - \nabla f_{\theta_2} (\theta) \| \leq \beta \| \theta_1 - \theta_2\| \quad \forall \theta \in \Theta
\end{align}
We refer to this condition as distribution continuity. 
%It can be thought of as a sensitivity constraint of the distribution-generating parameter. 
This condition was first introduced for the imitation learning problem by \cite{cheng2018convergence}. If $\beta = 0$ and $J_\tau(\theta) \neq 0$, this would correspond to a dynamics model where controls have no influence. 
We emphasize that the difference between the terms on the left-hand side lies entirely in the effect of  $p(\tau; \theta_1)$ and $p(\tau; \theta_2)$ being different in the objective. The constant $\beta$ can be interpreted as measuring the sensitivity of the loss to changes in the distribution-generating parameter.
Continuity in the distribution-generating parameter might also be expressed by directly assuming Lipschitz continuity of $p(\tau; \theta)$ but it turns out that this is as actually a stronger assumption because it implies (\ref{eq:regularity}),
assuming that $\sup_\tau \|\nabla J_\tau(\theta)\|$ is finite for all $\theta \in \Theta$. We discuss this in detail in Section~\ref{sec:examples}.

%the variation of the loss functions is related to the amount of change in the trajectory distributions induced by the sequence of policies. Ultimately these changes in trajectory distributions are dependent on the dynamics of the system. This can be seen by observing the bifunction form of the objective: $f_{\theta'}(\theta)$. Understanding the behavior of $f_{\theta'} (\theta)$ in terms of the distribution-generating parmaeter $\theta'$ can aid in characterizing the variation in order to obtain sublinear dynamic regret.

Continuity in the distribution-generating parameter is a characteristic of the system, which the user does not get to control explicitly. In contrast, the other assumptions made in this paper are largely up to the control of the user by imposing different choices of scalings and loss functions.

%We will show in the next section that a single assumption on the dynamics, introduced by Cheng and Boots \cite{cheng2018convergence}, can aid in characterizing the variation.

% In the spirit of \cite{cheng2018convergence} we can prove convergence or, at the very least, identify sufficient conditions for convergence using some knowledge of the sensitivity of the trajectory distribution. The next section will use these notions to show dynamic regret bounds and conditions for convergence for \textsc{Dagger}, imitation gradient and multiple imitation gradient. We will see that in some cases, depending on the choice of the measure of variation, conversions from known general dynamic regret rates to specific imitation learning convergence analyses can be straightforward.

\section{Main Results}\label{guarantees}

This section presents the primary theoretical results of the paper  for the infinite sample or deterministic case as in \citep{cheng2018convergence,ross2010reduction}. We aim to achieve the following:
\begin{enumerate} 
	\item To make rigorous the statements in the preceding section, showing that an on-policy algorithm's dynamic regret reveals important theoretical properties regarding convergence.
	\item To show that convergence and sublinear dynamic regret can be guaranteed for Algorithms~\ref{alg:ftl}-\ref{alg:omgd} under certain conditions on the sensitivity of the trajectory distribution.
\end{enumerate}
In particular, we show that it is possible for all algorithms to achieve $\frac{1}{N}R_D  = O(N^{-1})$ average dynamic regret (i.e. sublinear). Furthermore, because dynamic regret upper bounds static regret, these results suggest that we also improve static regret rates of \cite{ross2010reduction} by a logarithmic factor.

We first formally state the assumptions introduced in Section \ref{prelims}. We begin with common convex optimization assumptions on the loss in the evaluation parameter. Note that all gradients of $f$ in this section are taken with respect to the evaluation parameter, i.e. $\nabla f_{\theta'}(\theta)$ is the gradient with respect to $\theta$, not $\theta'$.

\begin{assumption}[Strong Convexity]\label{convexity}
	For all $\theta_1, \theta_2, \theta \in \Theta$, $\exists \alpha > 0$ such that
	\begin{align*}
	f_\theta(\theta_2) \geq f_\theta(\theta_1) + \langle \nabla f_\theta( \theta_1), \theta_2 - \theta_1 \rangle + \frac{\alpha}{2}\|\theta_1 - \theta_2\|^2.
	\end{align*}
\end{assumption}
\begin{assumption}[Smoothness and Bounded Gradient] \label{smoothness} 
	For all $\theta_1, \theta_2, \theta \in \Theta$, $\exists \gamma > 0$ such that
	\begin{align*}
	\|\nabla f_\theta (\theta_1) - \nabla f_\theta ( \theta_2) \| \leq \gamma \| \theta_1 - \theta_2\|
	\end{align*}
%	and $\exists G > 0$ such that $\forall \tau$, $\| \nabla J_\tau(\theta_1)\| < G$, which implies $\|\nabla f_\theta (\theta_1)\| \leq G$.
	and $\exists G > 0$ such that $\|\nabla f_\theta (\theta_1)\| \leq G$.
\end{assumption}
\begin{assumption}[Stationary Optimum]\label{relint}
	For all $\theta' \in \Theta$, $\theta^*$ is in the relative interior of $\Theta$ where $\theta^* = \argmin_{\theta \in \Theta} f_{\theta'}(\theta)$. That is, $\nabla f_{\theta'}(\theta^*) = 0$.
\end{assumption}
In practice the above conditions are not difficult to satisfy. For example, running DAgger with $l_2$-regularized linear regression, i.e. ridge regression, would simultaneously satisfy all three. Finally, we restate the regularity condition on the loss as a function of the distribution-generating parameter discussed in the previous section, which we can interpret as a measure of sensitivity of the loss.
\begin{assumption}[Distribution Continuity]\label{regularity}
	For all $\theta_1, \theta_2, \theta \in \Theta$, $\exists \beta> 0$ such that
	\begin{align*}
	\|\nabla f_{\theta_1}(\theta) - \nabla f_{\theta_2}(\theta)\| \leq \beta \|\theta_1 - \theta_2\|.
	\end{align*}
\end{assumption}
This assumption is a Lipschitz regularity on the gradients of the loss, but it is distinct from Assumption~\ref{smoothness} in that the Lipschitz continuity is in the distribution-generating parmeter, not the evaluation parameter. This is distinct from Assumption~\ref{smoothness}, where continuity is in the evaluation parameter. The first three assumptions describe conditions imposed on the actual loss function $J_\tau(\theta)$; Assumption~\ref{regularity} describes a constraint on the system dynamics by virtue of the trajectory distribution $p(\tau ; \theta)$.

%It is introduced as a form a prior knowledge of the dynamics as essentially a sensitivity constraint that implies we are assuming that small changes in the policy parameters guarantee small changes in the induced trajectory distributions. Such an assumption is not unlike assumptions that guarantee continuity of initial conditions studied in nonlinear systems \citep{khalil2002nonlinear}. In this analysis, we use Assumption \ref{regularity} because of its precedent in prior work \citep{cheng2018convergence}. We discuss this regularity condition in the next subsection and explore examples and precedents in  Section~\ref{sec:examples}.

\subsection{Solution Analysis under Dynamic Regret}
Let $\theta_n^* =\argmin_{\theta \in \Theta} f_n(\theta)$ be the optimal parameter at iteration $n$. This minimizer is unique because $f_n$ is strongly convex, but it is important to distinguish that this is not the same as the solution to the true objective $\min_{\theta \in \Theta} f_{\theta}(\theta)$. We begin with a result concerning a stability constant $\lambda := \frac{\beta}{\alpha}$. $\lambda$ represents the ratio of the distribution continuity and the strong convexity.

\begin{proposition}\label{runaway}
	Let $\theta, \hat \theta \in \Theta$ be two policy parameters inducing loss functions $f_{\theta}$ and $f_{\hat \theta}$. Let $\theta^*$ and $\hat \theta^*$ be the corresponding optimal policies on $f_\theta$ and $f_{\hat \theta}$. The following inequality holds:
	\begin{align*}
	\|\theta^* - \hat \theta^* \| \leq \lambda \|\theta - \hat \theta\|.
	\end{align*}
\end{proposition}

\begin{proof}
	Since $f_\theta$ is differentiable in the evaluation parameter, the first-order necessary optimality condition implies
	\begin{align*}
	\<\nabla f_\theta(\theta^*), \hat \theta^* - \theta^*\> \geq 0,
	\end{align*}
	and the same is true with $\theta$ and $\hat \theta$ switched. This, combined with strong convexity of $f_\theta$ implies
	\begin{align*}
	\frac{\alpha}{2}\|\theta^* - \hat \theta^*\|^2 \leq f_\theta(\hat \theta^*) -  f_\theta(\theta^*) .
	\end{align*}
	Applying strong convexity of $f_\theta$ a second time and the first-order optimality condition of $f_{\hat \theta}$ gives:
	\begin{align*}
	\alpha \|\theta^* - \hat \theta^*\|^2 &  \leq \< \nabla f_\theta(\hat \theta^*), \hat \theta - \theta\> \\
	& \leq \< \nabla f_\theta(\hat \theta^*)  - \nabla f_{\hat \theta}(\hat \theta^*) , \hat \theta - \theta\>.
	\end{align*}
	Bounding the right-hand side from above with the Cauchy-Schwarz inequality and rearranging terms, we have
	\begin{align*}
	\|\theta^* - \hat \theta^*\| \leq \frac{1}{\alpha} \|  \nabla f_\theta(\hat \theta^*)  - \nabla f_{\hat \theta}(\hat \theta^*) \|  \leq \frac{\beta }{\alpha} \| \theta - \hat \theta\|,
	\end{align*}
	where the last inequality follows from Assumption \ref{regularity}.
\end{proof}

%The proof is in the appendix.
There are a couple telling conclusions of this proposition. The first becomes clear when we take the parameters to be consecutive iterates. Let $\theta = \theta_n$ and $\hat \theta = \theta_{n+1}$ where $\{ \theta_n\}_{n=1}^N$ is a sequence of iterates from an arbitrary on-policy algorithms. The proposition suggests that when $\lambda < 1$, we know with certainty that $\|\theta_{n +1}^* - \theta_n^* \| < \|\theta_{n+1} -\theta_n\|.$ In other words, the optimal parameters cannot run away faster than the algorithm's iterates. We can even view this problem as a pursuit-evasion game, where, given Proposition~\ref{runaway}, there is hope that we can design an algorithm that catches the optimal parameter. This intuition is also consistent with the findings of prior work \citep{cheng2018convergence}, which shows that convergence of the $N$th policy can be guaranteed when $\lambda < 1$ for DAgger. 

Second, we can consider the self-mapping function $F:\Theta \mapsto \Theta$ defined by
$F(\theta) = \argmin_{\hat \theta \in \Theta} f_{\theta}(\hat \theta)$. If $\lambda < 1$, by Proposition~\ref{runaway}, $F$ is a contraction mapping on $\Theta$. Furthermore since $\Theta$ is compact, $F$ has a unique stationary point which is attainable by fixed-point iteration  \citep{banach1922operations} . This conclusion naturally leads to the following result.
\begin{corollary}\label{uniqueness}
	If $\lambda < 1$, then there exists a unique stationary point $\theta^\star \in \Theta$ such that the policy $\pi_{\theta^\star}$ satisfies
	\begin{align*}
	\mathbb E_{p (\tau; \pi_{\theta^\star})}  J_\tau( \pi_{\theta^\star}) = \min_{\theta \in \Theta} \mathbb E_{p (\tau; \pi_{\theta^\star})}  J_\tau( \pi_{\theta}).
	\end{align*}
\end{corollary}
\begin{proof}
By the Banach fixed-point theorem for the contraction mapping $F(\theta)$, there is a unique stationary point $\theta^\star \in \Theta$ satisfying $\theta^\star = F(\theta^\star)$. By the definition of $F$, $\theta^\star$ is the minimizer of $\{ \mathbb E_{p (\tau; \pi_{\theta^\star})}  J_\tau( \pi_{\theta}) \ : \ \theta \in \Theta \}$.
\end{proof}

This result is somewhat surprising. It implies that we cannot get stuck in other stationary points that may be worse than some globally optimal stationary point if $\lambda < 1$. Rather there is a single stationary point in $\Theta$ that performs optimally on its own distribution. We can also see that popular on-policy algorithms will preserve this stationary point. 

\begin{proposition}
	Suppose $\lambda < 1$. Let $\theta^\star$ be a stationary point in $\Theta$.
	For $n$ between $1$ and $N$, if $\theta_n = \theta^\star$, then $\theta_m = \theta^\star$ for all $n \leq m \leq N$ for Algorithms~\ref{alg:ftl}-\ref{alg:omgd}.
\end{proposition}
\begin{proof}
For DAgger (Algorithm \ref{alg:ftl}), if $\theta_n = \theta^\star$, then 
\begin{align*}
\theta^\star = \argmin_{\theta \in \Theta} \sum_{m = 1}^{n - 1} f_m(\theta).
\end{align*}By definition $\theta^\star = \argmin_{\theta \in \Theta} f_{\theta^\star}(\theta)$. Therefore $\theta^\star = \argmin_{\theta \in \Theta} \sum_{m = 1}^{n} f_m(\theta)$ and so $ \theta^\star$ is a stationary point for DAgger.

For Imitation Gradient (Algorithm \ref{alg:ogd}), if $\theta_n = \theta^\star$, then 
\begin{align*}
\theta_{n+1} = \argmin_{\theta \in \Theta}\quad \< \nabla f_{\theta^\star}(\theta^\star), \theta - \theta^\star\>  + \frac{1}{2}\| \theta - \theta^\star\|^2. 
\end{align*}
Because $\theta^\star$ minimizes $f_{\theta^\star}$, we have  $\< \nabla f_{\theta^\star}(\theta^\star), \theta - \theta^\star\> \geq 0$ by the first order optimality condition. Thus, the minimum, which is zero, is achieved at $\theta_{n+1} = \theta^\star$ and so $\theta^\star$ is stationary for this class of gradient-based algorithms. The same holds for Multiple Imitation Gradient (Algorithm \ref{alg:omgd}). Note that this is also true for mirror descent algorithms by substituting in the general Bregman divergence. 
\end{proof}
We emphasize that this notion of stationary points is an inherent property of on-policy algorithms and the imitation learning problem. It is not depednent on any modifications to the algorithms from what was introduced in the original works. However, as discussed in the previous section, convergence of static regret fails to reveal these properties.
We now illustrate the power of the dynamic regret analysis, showing that by achieving sublinear dynamic, we can converge to the stationary point and vice-versa.

\begin{theorem}
	Let $\{\theta_n\}_{n = 1}^N$ be a sequence of policy parameters generated by an arbitrary on-policy algorithm. Suppose that $\lambda < 1$. If $R_D(\theta_1, \ldots, \theta_N)$ is sublinear, there exists a subseqeunce of $\{\theta_n\}_{n = 1}^N$ that converges to the stationary point $\theta^\star$. Conversely, if $\{\theta_n\}_{n = 1}^N$ converges to $\theta^\star$, then $R_D(\theta_1, \ldots, \theta_N)$ is sublinear.

%	$R_D(\theta_1, \ldots, \theta_N)$ is sublienar, there exists a subseqeunce of $\{\theta_n\}_{n = 1}^N$ that converges to the stationary point $\theta^\star$.
	
%	Conversely, $\{\theta_n\}_{n = 1}^N$ converges to $\theta^\star$, then $R_D(\theta_1, \ldots, \theta_N)$ is sublinear.
\end{theorem}

\begin{proof}
	Because $\lambda < 1$, we know the stationary point exists and is unique.  Define $v(\theta) \define f_{\theta}(\theta) - \min_{\theta '} f_{\theta}(\theta')$ as a candidate potential  function. Note that $v(\theta) \geq 0$ for all $\theta \in \Theta$ and $v(\theta) = 0$ if and only if $\theta = \theta^\star$. 
	
	We first prove the latter direction. Suppose that $\theta_N \stackrel{N}{\rightarrow} \theta^\star$. Note that $v(\theta)$ is continuous at $\theta^\star$. Therefore, because $\lim_{N\rightarrow \infty} \theta_N = \theta^\star$, we know  $\lim_{N\rightarrow \infty} v(\theta_N) = v(\theta^\star) = 0$. This implies the desired result:
	\begin{align*}
	\lim_{N \rightarrow \infty} R_D(\theta_1, \ldots, \theta_N)  = \lim_{N \rightarrow \infty} \sum_{n = 1}^N v(\theta_n) = o(N).
	\end{align*}
	To prove the other direction, we are given that the sequence $\{v(\theta_n )\}_{n = 1}^N$ is sublinear in dynamic regret. This implies that there is a subsequence $\{v(\theta_{n_k})\}_{k = 1}^K$ such that  $\lim_{k \rightarrow \infty} v(\theta_{n_k}) = 0$ (e.g. take any monotonically decreasing subsequence).
	 Define the ball of radius $\epsilon$ centered at $\theta^\star$ as
	\begin{align*}
	B_\epsilon(\theta^\star) = \{ \theta \in \Theta \ : \ \|\theta - \theta^\star\| \leq \epsilon \}
	\end{align*}
	 Because $v$ is continuous at $\theta^\star$, for every $\epsilon > 0$, we can choose $\delta$ and the corresponding compact level set
	\begin{align*}
	\Omega_\delta \define  \left\{ \theta \in \Theta \ : \ v(\theta) \leq \delta \right\},
	\end{align*}
	which satisfies  $\Omega_\delta \subset B_\epsilon(\theta^\star)$. The level set can be made invariant by taking $K$ large enough such that $\forall k \geq K$, $v(\theta_{n_k}) < \delta$. This is possible because $v(\theta_{n_k})$ converges to zero. Since $\epsilon$ was arbitrary, we have that the sequence $\{\theta_{n_k}\}$ converges to $\theta^\star$.
\end{proof}

\begin{remark}
	If it is known that $R_D$ is sublinear, we can acquire a convergent sequence by taking
	\begin{align*}
	\hat \theta_N = \argmin_{\theta \in  \{\theta_n\}_{n = 1}^N} v(\theta).
	\end{align*} 
\end{remark}

The function $v(\theta)$ is reminiscent of a Lyapunov function. Indeed, the observation that sublinear dynamic regret implies convergence essentially mirrors the proof of Lyapunov's direct method for asymptotic stability \citep{khalil2002nonlinear,sastry2013nonlinear}. Generalizations of such functions also appear in literature on variational inequalities as we noted in \cite{cheng2019online}. In this field, they are referred to as merit functions, acting as proxies between the variational inequality problem and the other optimization problems with known solutions \citep{fukushima1996merit}. In our case, we use $v(\theta)$ as a merit function between between the dynamic regret problem and the fixed-point problem, which is reformulation of certain variational inequalities \citep{facchinei2007finite}. We leave extended discussion on this topic for Section~\ref{sec:col}.

So far, we have shown that achieving sublinear dynamic regret reveals not only that policies are performing optimally on their own distributions, but also that we can guarantee convergence to a unique fixed point when $\lambda < 1$. We now turn our attention to popular on-policy algorithms, showing that it is indeed possible for these algorithms to achieve sublinear regret under conditions on the loss function properties $\alpha$, $\beta$, and $\gamma$.

\subsection{DAgger}

We now introduce a dynamic regret corollary to Theorem 2 of \cite{cheng2018convergence}. For convenience, we restate their result.
\begin{theorem}[(\cite{cheng2018convergence}, Theorem 2)] 
	For $\lambda = \frac{\beta}{\alpha}$ and $n \geq 1$, it holds that
	\begin{align*}
	f_n(\theta_n) - f_n(\theta_n^*) \leq \frac{\left(\lambda e^{1-\lambda} G \right)^2}{2\alpha n^{2(1 - \lambda)}},
	\end{align*}
	and, if $\lambda < 1$, then $\{ \theta_n \}_{n=1}^\infty$ is convergent.
\end{theorem}
\begin{corollary}
	For DAgger  (Algorithm \ref{alg:ftl}), if $\lambda < 1$, then the average dynamic regret $\frac{1}{N}R_D$ tends towards zero in $N$ with rate $O(\max(N^{-1}, N^{2\lambda - 2}))$.
\end{corollary}
\begin{proof}
	The proof is immediate from the result of Theorem 2 of \cite{cheng2018convergence}. We have 
	\begin{align*}
	f_n(\theta_n) - f_n(\theta_n^*) \leq \frac{\left(\lambda e^{1-\lambda} G \right)^2}{2\alpha n^{2(1 - \lambda)}}.
	\end{align*}
	Summing from $1$ to $N$, we get \begin{align*}\sum_{n = 1}^N f_n(\theta_n) - \sum_{n = 1}^N f_n(\theta_n^*) & \leq \sum_{n = 1}^N \frac{\left(\lambda e^{1-\lambda} G \right)^2}{2\alpha n^{2(1 - \lambda)}}  \\
	& = O(\max(1, N^{2\lambda - 1})).\end{align*}
	Then the average dynamic regret  is \begin{align*}\frac{1}{N}R_D = O(\max(N^{-1}, N^{2\lambda - 2})),\end{align*} which approaches zero in $N$.
	\end{proof}

The corollary reveals that the convergence result for DAgger proved by Cheng and Boots can be reframed as a dynamic regret analysis. The rate is dependent on the stability constant $\lambda < 1$. The dynamic regret shows that policies generated from DAgger on average achieve local optimality and that for a sufficiently small $\lambda$ the regret grows no more than a finite amount.

\subsection{Imitation Gradient}
For the analysis of dynamic regret bounds for the Imitation Gradient algorithm, we require a slightly stronger condition that $\alpha^2 > 2 \gamma \beta$. Written another way, the condition is $2 \lambda < \psi$ where $\lambda$ is the stability constant and $\psi := \frac{\alpha}{\gamma}$ is the condition number of $f_n$. So we require that the problem is both stable and well-conditioned. The proof of this theorem, which can be found in the next section, makes use of the path variation.

\begin{theorem}\label{ogd_theorem}
	For Imitation Gradient  (Algorithm \ref{alg:ogd}), if $\lambda < 1$, $2 \lambda < \psi$ and $\eta = \frac{\alpha(\alpha^2 - 2\gamma \beta)}{2\gamma^2(\alpha^2 - \beta^2)}$, then the average dynamic regret $\frac{1}{N}R_D$ tends towards zero in $N$ with rate $O(N^{-1})$. Furthermore, $f_n(\theta_n) - f_n(\theta_n^*)$ converges to zero and the sequence of policies $\{\theta_n\}_{n=1}^\infty$ is convergent.
\end{theorem}

%The proof is in the appendix. 
Intuitively, the theorem states that if the conditions are met, the dynamic regret grows no more than a constant value. If we again interpret the variation as describing the difficulty of a problem, this theorem suggests that under the appropriate conditions, solving an imitation learning problem with Imitation Gradient algorithms is as easy as solving a general dynamic regret problem with path variation $V(\theta^*_{1:N}) = O(1)$. In other words, the equivalent dynamic regret problem is stationary in the limit: the optimal parameters cumulatively move no more than a finite distance. The reason is that the change in the loss functions is so closely tied to the policy parameters in imitation learning.

%For this proof, we rely on two Lemmas

\subsection{ Multiple Imitation Gradient}

We now present a similar dynamic regret rate for the Multiple Imitation Gradient algorithm. This theorem will make use of the squared path variation as opposed to the path variation. The squared path variation is especially amenable for a conversion from the standard dynamic regret bounds to interpretable rates for imitation learning, as we will see in the proof of the theorem.

%
%
%A straightforward conversion to characterize the measure of variation can be established by simply bounding from above the squared path variation by a quantity proportional to the dynamic regret using Assumption \ref{regularity}. We refer to this technique as establishing the reverse relationship between the measure of variation and the dynamic regret. To illustrate this conversion in detail, we present a proof sketch, which leverages the general dynamic regret rate for online multiple gradient descent first given by \cite{zhang2017improved}.
%
%\begin{lemma}[(\cite{zhang2017improved}, Theorem 3)]\label{zhang-theorem}
%	If Assumptions 1-3 hold and $\eta < 1/\gamma$ and $K = \lceil \frac{1/\eta + \alpha}{2\alpha}\log 4 \rceil$, then the following is true for online multiple gradient descent:
%	\begin{align*}
%	R_D(\theta_1, \ldots \theta_N) \leq 2\gamma S(\theta^*_{1:N}) + \gamma \|\theta_1 - \theta_1^*\|^2.
%	\end{align*}\end{lemma}

%From this lemma, a specific result for imitation learning can obtained in a straightforward manner by incorporating the regularity and the strong convexity of the loss functions. 

\begin{theorem}\label{omgd_theorem}
	For Multiple Imitation Gradient (Algorithm \ref{alg:omgd}), if $\lambda < 1$, $ \lambda \log 4 < \psi^{3/2}$, $\eta < \min\left\{1/\gamma, \frac{\alpha^{5/2} - \gamma^{3/2} \beta \log 4}{2\gamma^{3/2}\alpha \beta \log 4}\right\}$, and $K = \lceil \frac{1/\eta + \alpha}{2\alpha}\log 4 \rceil$ then the average dynamic regret $\frac{1}{N}R_D$ tends towards zero in $N$ with rate $O(N^{-1})$. Furthermore,  $f_n(\theta_n) - f_n(\theta_n^*)$ converges to zero and the sequence of policies $\{\theta_n\}_{n=1}^\infty$ is convergent.
\end{theorem}

The above theorem demonstrates that again a constant upper bound on the dynamic regret can be obtained with this algorithm. Interestingly, the conditions sufficient to guarantee convergence are different. Instead of requiring $2\lambda < \psi$, this theorem requires $ \lambda \log 4 < \psi^{3/2}$. While $\log 4 < 2$, we also have that $\psi^{3/2} \leq \psi$ because the condition number, which is the ratio of $\alpha$ and $\gamma$, is at most 1.

The implication of this observation is that there is a trade-off. For the Multiple Imitation Gradient algorithm, we can guarantee $O(N^{-1})$ average regret for higher values of $\lambda$, i.e. more sensitive systems, but we must then require that our loss functions are better conditioned. Conversely, the Imitation Gradient algorithm achieves a similar guarantee using losses with lower condition numbers, but $\lambda$ must be smaller.

\subsection{Connections to Continuous Online Learning and the Variational Inequality Problem}\label{sec:col}

The results in this section are closely related to the problem of continuous online learning, which we introduced in \cite{cheng2019online}. Indeed, the assumptions that we impose on the bifunction $f_{\theta'}(\theta)$ fall within this framework. One of the key features of continuous online learning is that achieving dyanmic regret in these settings is equivalent to solving certain variational inequalities, fixed-point problems, and equilibrium problems, paralleling our results in the previous section. Here, we review the variational inequality problem and its connections to on-policy imitation learning. A detailed overview can be found in \cite{facchinei2007finite}.

The variational inequality aims to find solutions to the following problem: for a map $F : \Theta \mapsto \R^d$, find $\theta \in \Theta$ such that
\begin{align*}
\< F(\theta), \theta' - \theta\> \geq 0 \quad  \forall \theta' \in \Theta.
\end{align*}
The problem is written as $\vi(\Theta, F)$ and the set of solutions is written as $\sol(\Theta, F)$. The map $F$ is said to be monotone if 
\begin{align*}
\< F(\theta) - F(\theta'), \theta - \theta' \>  \geq 0  \quad  \forall \theta, \theta' \in \Theta.
\end{align*}
It is said to be $\zeta$-strongly monotone if
\begin{align*}
\< F(\theta) - F(\theta'), \theta - \theta' \>  \geq \zeta \| \theta - \theta'\|^2  \quad  \forall \theta, \theta' \in \Theta.
\end{align*}
In literature on variational inequalities, strongly monotone functions conduce to efficient algorithms for finding solutions.

In the context of continuous online learning, it can be seen by first order optimality conditions that $f_\theta(\theta) - \min_{\theta' \in \Theta} f_\theta(\theta') = 0$ if and only if $\theta$ is a solution to the variational inequality with $F(\theta) \define \nabla f_\theta(\theta)$. In \cite{cheng2019online}, we proved that if $\alpha > \beta$, $f_{\theta}(\cdot)$ is $\alpha$-strongly convex, and $\nabla f_{\cdot}(\theta)$ is $\beta$-Lipschitz, then $\nabla f_\theta(\theta)$ is a $(\alpha - \beta)$-strongly monotone map.
There are two important implications that are direct results of standard variational inequalities. First, because $\nabla f_\theta(\theta)$ is strongly monotone and $\Theta$ is compact, there is a unique solution to $\vi(\Theta, \nabla f)$. Second, there exist efficient algorithms to find this solutions such as the Basic Projection Algorithm \citep{facchinei2007finite}
\begin{align*}
\theta_{n+1} = P_\Theta( \theta_{n} - \eta F(\theta_n)),
\end{align*}
which exactly corresponds to online gradient descent. There are a number of other properties of variational inequalities, equilibrium problems, and fixed-point problems that are potentially of interest for future work in the on-policy imitation learning problem.

\section {Complete Proofs of Algorithm Guarantees}
 We now present the proofs of the theorems for both gradient-based methods. Both will make heavy use of the following well-known result on the strong convexity of a function \citep{hazan2014beyond}.
\begin{lemma}\label{lemma:stronglyconvex-simple}
	The following holds for all $\theta \in \Theta$ and $\theta^* = \argmin_{\theta'} f(\theta')$:
	\begin{align*}
	f(\theta) - f(\theta^*) \geq \frac{\alpha}{2}\|\theta - \theta^*\|^2.
	\end{align*}
\end{lemma}

\subsection{Proof of Theorem~\ref{ogd_theorem}}
	Before directly proving this theorem, we establish several supporting results based on the path variation.
	\begin{lemma}\label{path_var_bound}
		For a sequence of predictions made by Imitation Gradient $\{\theta_n\}_{n=1}^N$ and a sequence of optimal parameters $\{\theta^*_n\}_{n=1}^N$, the following inequality holds on the path variation:
		\begin{align*}
		V(\theta^*_{1:N}) \leq \eta \frac{\beta\gamma}{\alpha} \sum_{n = 1}^N \| \theta_n - \theta^*_n \|.
		\end{align*}
	\end{lemma}
	\begin{proof}
		From Proposition \ref{runaway}, we have \begin{align*}\|\theta_{n+1}^* - \theta_n^*\|&  \leq \frac{\beta}{\alpha}\| \theta_{n+1} - \theta_n \|\\ & \leq \frac{\beta}{\alpha}\| \eta \nabla f_n(\theta_n) \| \\
		& = \eta \frac{\beta}{\alpha} \|\nabla f_n(\theta_n) - \nabla f_n(\theta_n^*) \| \\ 
		& \leq \eta \frac{\beta \gamma}{\alpha} \|\theta_n - \theta_n^*\|,
		\end{align*}
		where the final equality uses Assumption \ref{relint} and the final inequality uses Assumption \ref{smoothness}. The result follows immediately by summing.
		\end{proof}

	\begin{lemma}\label{series_bound}
		Let $\rho = (1 - \alpha \eta + \gamma^2 \eta^2)^{1/2} > 0$. The following inequality holds for Imitation Gradient
		\begin{align*}
		\sum_{n = 1}^N \|\theta_n - \theta_n^*\| \leq \|\theta_1 - \theta_1^*\| + \sum_{n = 1}^N \rho \|\theta_n - \theta_n^*\| + V(\theta_{1:N}^*).
		\end{align*}
	\end{lemma}
	\begin{proof} First note that $\rho$ is always non-negative because, for any positive choice of $\eta$, we have $\gamma \geq \alpha$ by definition. 
		By strong convexity we have the following: $
		0 \leq 2(f_n(\theta_n) - f_n(\theta_n^*)) 
		\leq 2 \langle \nabla f_n (\theta_n), \theta_n - \theta_n^* \rangle - \alpha \|\theta_n^* - \theta_n\|^2.
		$
		By the update rule given in the algorithm:
		\begin{align*}
		\| \theta_{n+1} - \theta_n^*\|^2 & \leq \|\theta_n - \eta \nabla f_n(\theta_n) - \theta_n^*\|^2 \\
		& = \|\eta \nabla f_n(\theta_n)\|^2 + \|\theta_n - \theta_n^*\|^2  \\
		& \quad - 2 \eta \langle\nabla f_n (\theta_n),  \theta_n - \theta_n^* \rangle. \numberthis \label{update_ineq}
		\end{align*}
		Rearranging the terms in (\ref{update_ineq}) and combining with the very first inequality, we arrive at the following:
		\begin{align*}
		\| \theta_{n+1} - \theta_n^*\|^2 \leq (1 - \alpha \eta) \|\theta_n - \theta_n^*\|^2 + \|\eta \nabla f_n(\theta_n)\|^2.
		\end{align*}
		Using Assumption \ref{smoothness}, the smoothness of $f_n$, and the fact that $\nabla f_n(\theta_n^*) = 0$:
		\begin{align*}
		\| \theta_{n+1} - \theta_n^*\|^2 & \leq \|\theta_n - \theta_n^*\|^2 - \alpha \eta \|\theta_n - \theta_n^*\|^2 \\
		& \quad + \eta^2 \|\nabla f_n(\theta_n) - \nabla f_n(\theta_n^*)\|^2 \\
		& \leq \left(1 - \alpha \eta + \gamma^2 \eta^2\right) \|\theta_n - \theta_n^*\|^2. \numberthis \label{next_bound}
		\end{align*}
		Let $\rho = \left(1 - \alpha\eta + \gamma^2 \eta^2 \right)^{1/2}$. Following from \cite{mokhtari2016online}, we consider the series
		\begin{align*}
		\sum_{n = 1}^N \|\theta_n - \theta_n^*\|
		& \leq \|\theta_1 - \theta_1^*\| + \sum_{n = 2}^N \| \theta_n - \theta_{n - 1}^* \| + V(\theta^*_{1:N}) \\
		& \leq \|\theta_1 - \theta_1^*\| + \sum_{n = 1}^{N} \rho \| \theta_{n} - \theta_{n}^* \| + V(\theta^*_{1:N}),
		\end{align*} 
		where the first line uses the triangle inequality and the definition of the path variation and the third line uses (\ref{next_bound}).
		\end{proof}

	\begin{proof}[Proof of Theorem \ref{ogd_theorem}]
		We begin by bounding the result from Lemma \ref{series_bound} above by Lemma \ref{path_var_bound}:
		\begin{align*}
		\sum_{n = 1}^N \|\theta_n - \theta_n^*\| & \leq \|\theta_1 - \theta_1^*\|  + \left(\rho + \eta \frac{\beta\gamma}{\alpha}\right) \sum_{n = 1}^N \|\theta_n - \theta_n^*\|.
		\end{align*}
		By rearranging the terms and bounding by the diameter of $\Theta$:
		\begin{align}\label{ogd-parameter-diff}
		\sum_{n = 1}^N \|\theta_n - \theta_n^*\|  \leq \frac{D}{1 - \rho - \eta \frac{\beta\gamma}{\alpha}}.
		\end{align}
		Choosing $\eta = \frac{\alpha(\alpha^2 - 2\beta \gamma)}{2\gamma^2(\alpha^2 - \beta^2)}$ ensures that $\left(1 - \rho - \eta \frac{\beta\gamma}{\alpha}\right)$ is positive. By the $G$-Lipschitz continuity of $f_n$, we have
		\begin{align*}
		\sum_{n = 1}^N f_n(\theta_n) - \sum_{n = 1}^N f_n(\theta_n^*)  \leq \frac{GD}{1 - \rho - \eta \frac{\beta\gamma}{\alpha}},
		\end{align*}
		and so $R_D(\theta_1, \ldots, \theta_N) = O(1)$. Therefore, $\frac{1}{N} R_D = O(1/N)$ which approaches zero.
		
		Interpreting the dynamic regret as a monotone partial sum, we have bounded $\sum_{n=1}^\infty f_n(\theta_n) - f_n(\theta_n^*)$ by a constant, meaning that the series converges by the monotone convergence theorem. Because the elements of the partial sum, $f_n(\theta_n) - f_n(\theta_n^*)$, are non-negative, they must then converge to zero.
		The proof of convergence of the parameters takes a similar approach.
		From the proof of Lemma \ref{path_var_bound}, we have 
		\begin{align*}
		\frac{1}{\eta\gamma}\| \theta_{n+1} - \theta_n \| \leq \| \theta_n - \theta_n^*\|,
		\end{align*}
		The series $\sum_{n = 1}^\infty  \| \theta_n - \theta_n^*\|$ converges due to (\ref{ogd-parameter-diff}). Thus the series $\sum_{n = 1}^\infty  \| \theta_{n+1} - \theta_n\|$ converges.
		Since the series of distances between consecutive elements converges, we have that the sequence $\{\theta_n\}_{n = 1}^\infty$ converges in the limit.
		\end{proof}

\subsection{Proof of Theorem~\ref{omgd_theorem}} For this proof, we again require two technical lemmas modified from \cite{zhang2017improved} which are proved in the appendix for completeness.
\begin{lemma}\label{lemma:decreasing-distance}
	For online multiple gradient descent, let $\theta'$ be the current parameter played by the algorithm at any iteration, $\hat \theta \define P_\Theta(\theta' - \eta \nabla f_n(\theta'))$. It holds that
	\begin{align*}
	\|\hat \theta - \theta_n^*\|^2 \leq \left(1 - \frac{2\alpha}{1/\eta + \alpha } \right) \|\theta' - \theta_n^* \|^2.
	\end{align*}
\end{lemma}
\begin{lemma}[(\cite{zhang2017improved}, Theorem 3)]\label{zhang-theorem}
	If $\eta < 1/\gamma$ and $K = \lceil \frac{1/\eta + \alpha}{2\alpha}\log 4 \rceil$, then the following is true for online multiple gradient descent:
	\begin{align*}
	R_D(\theta_1, \ldots \theta_N) \leq 2\gamma S(\theta^*_{1:N}) + \gamma \|\theta_1 - \theta_1^*\|^2.
	\end{align*}\end{lemma}
\begin{proof}[Proof of Theorem \ref{omgd_theorem}] We begin by upper bounding the squared path variation using Proposition \ref{runaway} and Assumption \ref{regularity}:

\begin{align*}
\|\theta^*_{n+1} - \theta_n^*\| & \leq \frac{\beta}{\alpha}\| \theta_{n+1} - \theta_n \| \\
& \leq \frac{\beta\eta}{\alpha}\| \nabla f_n(\theta_n^1) + \ldots + \nabla f_n(\theta_n^K) \| \\
& \leq \frac{\beta\eta}{\alpha} \sum_{j = 1}^K \|\nabla f_n(\theta_n^j)\| \\
& = \frac{\beta\eta}{\alpha} \sum_{j = 1}^K \|\nabla f_n(\theta_n^j) - \nabla f_n(\theta_n^*)\|
\end{align*}
Next we apply Lemma \ref{lemma:decreasing-distance}:
\begin{align*}
\|\theta^*_{n+1} - \theta_n^*\| & \leq \frac{\beta\eta\gamma }{\alpha} \sum_{j = 1}^K \|\theta_n^j - \theta_n^*\| \\
& \leq \frac{\beta\eta\gamma }{\alpha} \|\theta_n - \theta_n^*\| \sum_{j = 1}^K \left(1 - \frac{2\alpha}{1/\eta + \alpha}\right)^j \\
& \leq \frac{\beta\eta\gamma K}{\alpha} \|\theta_n - \theta_n^*\|
\end{align*}
Summing over all $n$, we have
\begin{align*}
S(\theta^*_{1:N}) \leq \left(\frac{\beta \eta \gamma K}{\alpha}\right)^2 \sum_{n = 1}^N \|\theta_n - \theta_n^*\|^2
\end{align*}
Then by substituting into Lemma \ref{zhang-theorem} and bounding the quantity from above using the strong convexity of $f_n$, we have
\begin{align*}
R_D & \leq \gamma \|\theta_1 - \theta_1^*\|^2 + \frac{4\beta^2 \eta^2 \gamma^3 K^2}{\alpha^3} \sum_{n = 1}^N f_n(\theta_n) - f_n(\theta_n^*) \\
%\gamma \|\theta_1 - \theta_1^*\|^2 +  \frac{2\beta^2 \eta^2 \gamma^3 K^2}{\alpha^2} \sum_{n = 1}^N \|\theta_n - \theta_n^*\|^2 \\
%& \leq \gamma \|\theta_1 - \theta_1^*\|^2 + \frac{4\beta^2 \eta^2 \gamma^3 K^2}{\alpha^3} \sum_{n = 1}^N f_n(\theta_n) - f_n(\theta_n^*) \\
& \leq \frac{\gamma \|\theta_1 - \theta_1^*\|^2}{1 - \frac{4\beta^2 \eta^2 \gamma^3 K^2}{\alpha^3}} \leq \frac{\gamma D^2}{1 - \frac{4\beta^2 \eta^2 \gamma^3 K^2}{\alpha^3}}.
\end{align*}
For $\eta < \min\left\{1/\gamma, \frac{\alpha^{5/2} - \gamma^{3/2} \beta \log 4}{2\gamma^{3/2}\alpha \beta \log 4}\right\}$ and $K > \frac{1/\eta + \alpha}{2\alpha}\log 4$, the denominator is positive. So the regret is again $O(1)$.

The proof of convergence of  $f_n(\theta_n) - f_n(\theta_n^*)$ to zero is identical to that of Theorem $\ref{ogd_theorem}$. The proof of convergence of $\{\theta_n\}_{n = 1}^\infty$ is similar. An intermediate result of Lemma \ref{zhang-theorem} yields
\begin{align}
\|\theta_{n+1} - \theta_n^*\| \leq \frac{1}{2}\|\theta_n - \theta_n^*\|.
\end{align}
Then, as before, we bound the sum of the differences between policy parameters and optima, but without the square:
\begin{align*}
\sum_{n = 1}^N \|\theta_{n} - \theta_n^*\| \leq 2\|\theta_1 - \theta_1^*\| + 2V(\theta^*_{1:N})
\end{align*}
By bounding the path variation from above using Lemma \ref{lemma:decreasing-distance} and again rearranging terms, a constant upper bound on $\sum_{n = 1}^N \|\theta_{n} - \theta_n^*\|$ is established in the same way as Theorem \ref{ogd_theorem} except we require $1 > \frac{2\beta \eta\gamma K}{\alpha}$. This is satisfied with the same condition on $\eta$. The rest of proof proceeds exactly in the same way as Theorem \ref{ogd_theorem}. Therefore with Multiple Imitation Gradient, $\{\theta_n\}_{n=1}^N$ converges in the limit. \end{proof}

\section{Systems with Distribution Continuity}\label{sec:examples}

Fundamentally, the preceding results for on-policy imitation learning and core idea behind continuous online learning rely on the supposition that there is Lipschitz continuity in $\nabla f_{\theta'}(\theta)$ with respect to $\theta'$, the distribution-generating parameter. However, it may not be immediately clear when this condition is satisfied in scenarios of interest. In this section, we discuss several examples of well-studied MDPs that satisfy this condition.

\subsection{Linear Dynamical Systems}
We can show that the ridge regression loss on a linear dynamical system satisfies Assumption \ref{regularity}. We consider the following system:
\begin{align*}
s_{t+1} = A s_t + Bu_t,
\end{align*}
where $A \in \R^{m \times m}$ and $B \in \R^{m\times 1}$ and the fixed initial state is $s_1 \in \R^m$.
The policies give linear feedback with $u_t = \theta^\top s_t$ with $\theta \in \R^{m}$, so the closed loop dynamics are
\begin{align*}
s_{t+1} = (A  + B \theta^\top) s_t.
\end{align*}
We consider a supervisor policy that is affine in the state $Ks_t + k$. Over a full trajectory, the loss function with distribution-generating parameter $\theta'$ and evaluation parameter $\theta$ is  
\begin{align*}
f_{\theta'}(\theta) = \left\| \begin{bmatrix}
s_1^\top \\ \vdots \\ s_{T -1 }^\top
\end{bmatrix} \theta -  \begin{bmatrix}(Ks_1 + k)^\top \\ \vdots \\
(Ks_{T - 1} + k)^\top
\end{bmatrix} \right\|^2 + \frac{\alpha}{2}\|\theta\|^2
\end{align*}
where $s_t = (A + B (\theta')^\top)^{t - 1}s_1$. The derivative is then 
\begin{align*}
\nabla f_{\theta'}(\theta) = \begin{bmatrix}
s_1s_1^\top \\ \vdots \\ s_{T- 1} s_{T -1 }^\top
\end{bmatrix} \theta -  \begin{bmatrix}s_1(Ks_1 + k)^\top \\ \vdots \\
s_{T - 1}(Ks_{T - 1} + k)^\top
\end{bmatrix}  + \alpha\theta
\end{align*} 
Because of the definition of $s_t$, we can see that $\nabla f_{\theta'}( \theta)$ is a polynomial map in $\theta'$. Therefore, on the compact set $\Theta$, its derivative is continuous and bounded, so $\nabla f_{\theta'}(\theta)$ is Lipschitz continuous.

The supervisor policy can be polynomial in the state and the result is the same due to composition of polynomials. We can also extend this to multiple input systems but, notationally, this requires redefining $\Theta$ as a subset of matrices.

\subsection{Lipschitz Markov Decision Processes}

Another problem class that satisfies the conditions of Assumption~\ref{regularity} is ``Lipschitz MDPs" \citep{bertsekas1975convergence,hinderer2005lipschitz,pirotta2015policy,asadi2018lipschitz}. This category of Markov decision processes studied in reinforcement learning theory is restricted to exhibit smoothness in the transition dyanmics and the reward. 
Before discussing in detail the formal Lipschitz MDP setting, we first return to an intuitive example that illustrates this smoothness of dynamics from Section~\ref{sec:dyn}.

The left-hand side of the inequality in Assumption \ref{regularity}, can be expanded as
\begin{align*}
\| \nabla f_{\theta_1}(\theta) - \nabla f_{\theta_2}(\theta) \| & = \| \int  \nabla J_{\tau}(\theta) ( p(\tau ; \pi_{\theta_1}) - p(\tau ; \pi_{\theta_2}) ) d\tau  \| \\
& \leq \sup_\tau \ \| \nabla J_\tau(\theta)\| \int | p(\tau ; \pi_{\theta_1}) - p(\tau ; \pi_{\theta_2}) | d\tau \\
& = \sup_\tau \ 2 \| \nabla J_\tau(\theta)\| d_{TV} \left(p( \tau ; \pi_{\theta_1}), p( \tau ; \pi_{\theta_2})\right)
\end{align*}
where $d_{TV}(p, q)$ denotes the total variation between distributions $p$ and $q$. If $\sup_\tau \| \nabla J_\tau(\theta)\|$ is finite, then Assumption \ref{regularity} depends only on the condition that small changes in the policy parameters induce small changes in the trajectory distributions with respect to total variation distance. That is, if $p( \tau | \pi_{\theta_2})$ is $L_{TV}$-Lipschitz
\begin{align*}
d_{TV} \left(p( \tau | \pi_{\theta_1}), p( \tau | \pi_{\theta_2})\right) \leq L_{TV} \| \theta_1 - \theta_2\|,
\end{align*}
then Assumption \ref{regularity} is satisfied with 
%$\beta = \sup_\tau \| \nabla J_\tau(\theta)\| L_{TV}$.
\begin{align*}\beta = \sup_\tau \ 2\| \nabla J_\tau(\theta)\| L_{TV}.
\end{align*}
%\begin{align*}
%F(\theta) = \argmin_{\theta'} f_\theta(\theta') \\
%\|F(\theta) - F(\hat \theta)\| \leq \lambda \| \theta - \hat \theta \| \\
%\blue{\theta^\star} = F(\blue{\theta^\star})
%\end{align*}

This condition is similar to differentiability assumptions in reinforcement learning \citep{schulman2015trust}. While it conveys intuition, the Lipschitz continuity of the trajectory distribution in total variation distance is a higher level assumption than what is typically considered in Lipschitz MDPs. However, it can be derived from lower level assumptions. In our discussion of this setting, we will follow the assumptions of \cite{pirotta2015policy}, who consider a infinite horizon ($T \rightarrow \infty$) and $\gamma$-discounted version of the MDP $\M$ for $\gamma \in (0, 1)$.

The Lipschitz MDP has regularity in the transition dynamics $p(\cdot | s, u)$ and the reward function. In the imitation learning setting, we do not assume access to a reward function but we can view the loss $\ell(s, \theta)$ in its place.
Let $d_{\S, \U}((s, u), (s', u')) = d_\S(s, s') + d_\U(u, u')$ denote the sum of distance metrics on the state set $\S$ and control set $\U$. Lipschitz MDPs assume there exists $L_p \geq 0$ such that 
%\begin{align*}
%\|\nabla \ell(u, u^*) - \nabla \ell(u', u^*)\| & \leq L_\ell d_{\U} (u, u'),
%\end{align*}
%and
\begin{align*}
\K(p(\cdot | s, u), p(\cdot | s', u')) & \leq L_p d_{\S, \U} ((s, u), (s', u')).
\end{align*}
where $\K$ is the Kantorovich distance,  defined for probability measures $p$ and $q$ as
\begin{align*}
\K(p, q) \define  \sup \left\{ \int f(x) (dp - dq) \ : \ f \text{ is 1-Lipschitz cts.}  \right\}
\end{align*}
Due to its generality, the Kantorovich distance (a specific case of the Wasserstein distance) is used extensively in literature on Lipschitz MDPs in place of the total variation, but the choice of distribution distance does not ultimately affect Assumption \ref{regularity}.
In addition, we assume that we have chosen a loss function $\ell$ such that the partial derivatives with respect to $\theta$ are $L_\ell$-Lipschitz:
\begin{align*}
\left| \frac{\partial \ell(s, \theta)}{\partial \theta_i} -  \frac{\partial \ell(s, \theta')}{\partial \theta_i'} \right| \leq L_\ell\|\theta - \theta'\|
\end{align*}

We consider an alternative, but equivalent form of the imitation learning objective. Rather than defining a trajectory distribution, we define an average state distribution $\delta_\theta$ under $\pi_\theta$ such that 
%$J_\tau(\theta) = \E_{(u, u^*, s) \sim \delta_\pi} \ell(u, u^*)$. 
$\E_{\tau\sim p(\tau | \theta)}J_\tau(\theta) = \E_{s \sim \delta_\theta} \ell(s, \theta)$. This is a standard reformulation \citep{ross2010reduction}.
Lipschitz continuity of $\delta_\theta$ in the Kantorovich distance was shown by \cite{pirotta2015policy}.
\begin{lemma}[\citep{pirotta2015policy}] 
	If $\pi_\theta(\cdot | s)$ is $L_\S$-Lipschitz in $s \in \S$ and Lipschitz in $\theta \in \Theta$ and $\gamma L_p (1 + L_\S ) < 1$, then there exists a finite $L_\delta > 0$ such that
	\begin{align*}
	\K(\delta_{\theta_1}, \delta_{\theta_2}) \leq L_\delta \| \theta_1 - \theta_2\|.
	\end{align*}
\end{lemma}
Then we can show that Assumption \ref{regularity} is satisfied when the conditions of the previous lemma are satisfied:
\begin{align*}
\| \nabla f_{\theta_1}(\theta) - \nabla f_{\theta_2}(\theta) \| & = \| \int \nabla \ell(s, \theta) (d\delta_{\theta_1} - d\delta_{\theta_2})  \|  \\
& \leq \sum_{i = 1}^d \left| \int  \frac{\partial \ell}{\partial \theta_i} (d\delta_{\theta_1} - d\delta_{\theta_2}) \right| \\
& \leq d L_\ell \K( \delta_{\theta_1}, \delta_{\theta_2} ) \\
& \leq d L_\ell L_\delta \| \theta_1 - \theta_2\|.
\end{align*}
In this case, we observe that $\beta = d L_\ell L_\delta$. This upper bound is fairly loose and we conjecture that the additional dimension dependence is more pessimistic than necessary.

\section{Adaptive On-Policy Regularization}

We now apply these theoretical results to motivate an adaptive regularization algorithm to help ensure convergence. As noted by \cite{cheng2018convergence}, regularization can lead to convergent policies for DAgger.

\begin{figure}

	\begin{tabular}[t]{@{} p{.5\linewidth} @{}}
		\begin{algorithm}[H]
			\caption{Adaptive On-Policy Regularization (\textsc{Aor})}
			\label{alg:aopr}
			\begin{algorithmic}
				\STATE {\bfseries Input:} Initial parameters $\theta_1$,
				\STATE ~~~~~~~~~~ Max iterations $N$,
				\STATE ~~~~~~~~~~ Initial regularization $\hat \alpha_1$.
				\FOR{$n=1$ {\bfseries to} $N-1$}
				\STATE Roll out $\theta_n$ and collect trajectory $\tau_n$.
				\STATE Observe loss $f_n(\theta) = f_{\theta_n}(\theta)$ from feedback on $\tau_n$.
				\STATE $\theta_{n+1} \gets \textsc{Update}(\theta_n, \hat \alpha_n)$.
				\STATE $\hat \lambda_n \gets \textsc{Estimate}(f_1, \ldots, f_n)$.
				\STATE $\hat \alpha_{n+1} \gets \textsc{Tune}(\hat \alpha_n, \hat \lambda_n)$.
				
				% 		t\hat \lambda_n \hat \alpha_n + (1 - t) \hat \alpha_n$.
				\ENDFOR
			\end{algorithmic}
		\end{algorithm}
	\end{tabular}

%	\vspace{.5em}
	\caption*{Adaptive On-Policy Regularization adaptively increases $\alpha$, the regularization parameter for strong convexity, to stabilize any regularized on-policy algorithm.}
%	\vspace{-2.5em}
\end{figure}

%In DAgger, Imitation Gradient, and Multiple Imitation Gradient, 
In all the theoretical results,
a key sufficient condition is that $\lambda < 1$, meaning that the strong convexity constant $\alpha$ must be greater than the regularity constant $\beta$. While $\beta$ is a fixed property of the dynamics, $\alpha$ is largely controllable by the user and robot. $\alpha$ can be increased by increasing the regularization of the supervised learning loss. By Proposition \ref{runaway}, a lower bound on $\lambda$ may be estimated by finding the ratio of the distance between optimal parameters and the distance between policy parameters: \begin{align*}\hat \lambda = \frac{\| \theta^*_{n+1}  - \theta^*_n \|}{\| \theta_{n+1} - \theta_n\|}.\end{align*}

\begin{figure}[h]
	%	\vspace{-1.5em}
	\centering
	\includegraphics[width=4in]{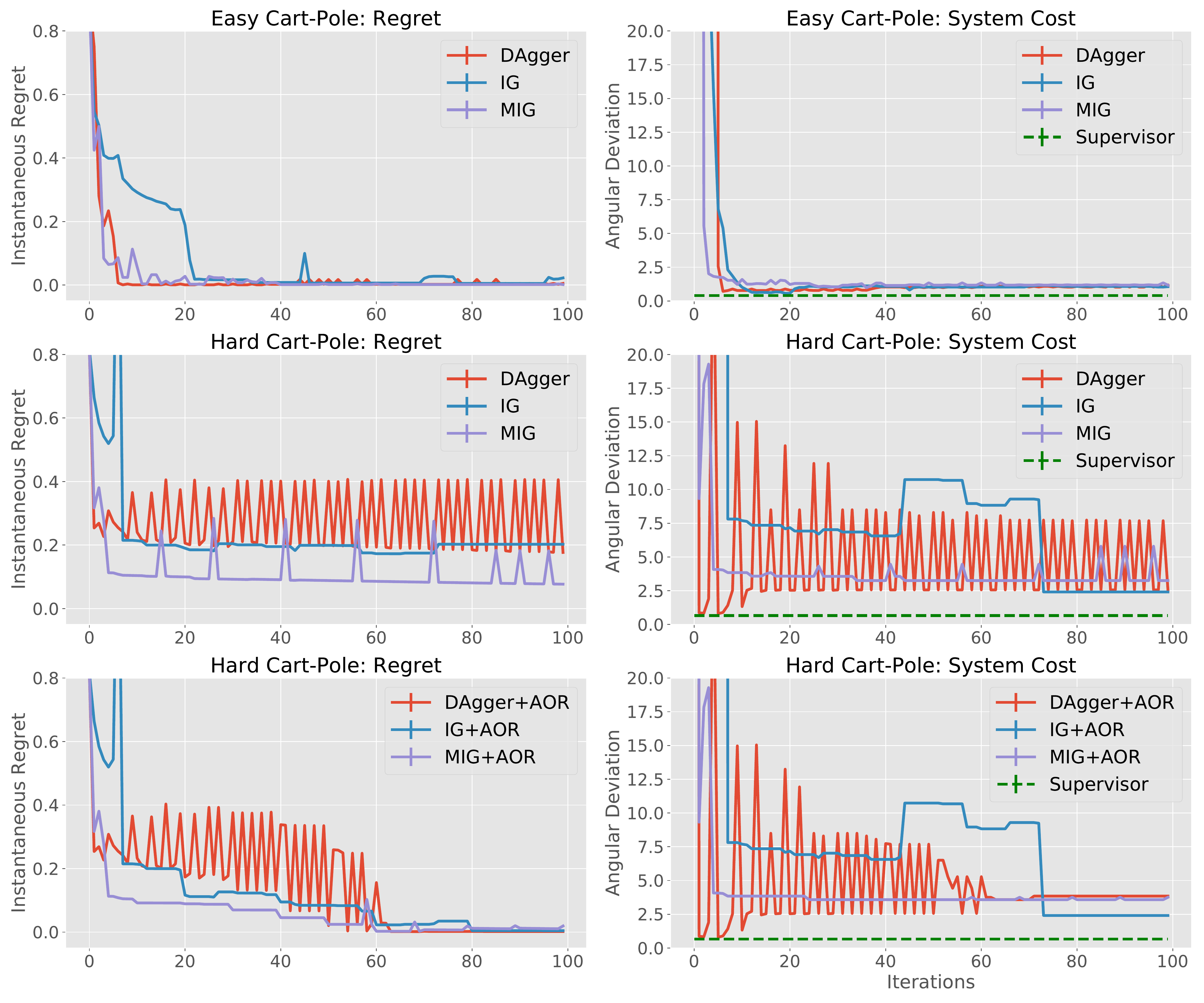}
	\caption{ For the cart-pole balancing task, regret is shown in the right column while true system cost, measured as angular deviation, is shown in the left column. All three on-policy algorithms without adaptive regularization are shown on both the easy (top row) and hard (middle, bottom rows) settings of cart-pole. The result shows existence of a system is difficult enough to induce unstable learning curves. The bottom row shows the algorithms again on the hard setting but using \textsc{Aor}. The dynamic regret tends towards zero and the chattering in DAgger and MIG is reduced.}
	\label{fig:cartpole}
	%	\vspace{-1.5em}
\end{figure}

Therefore, we can propose an adaptive algorithm to compute a new regularization term at each iteration $n$. One caveat of adaptively updating $\alpha$ is that we do not want it to be too large. While the regret will converge, the policy performance can suffer as the regularization term will dominate the loss function and thus the convergent solution will simply be the solution that minimizes the regularizer. This subtlety and the theoretical motivation in the previous sections are the basis for Algorithm \ref{alg:aopr}, which we call Adaptive On-Policy Regularization, an algorithm for making conservative updates to $\alpha$ that can be applied to any regularized on-policy algorithm. At each iteration, the policy is updated according to a given on-policy algorithm such as DAgger using subroutine $\textsc{Update}(\theta_n, \hat \alpha_n)$ which depends on the current regularization. Then $\hat \lambda_n$ is computed with subroutine $\textsc{Estimate}$. We use a mean of observed $\lambda$ values over iterations. Finally, the $\hat \alpha_{n+1}$ is tuned based on $\hat \lambda_{n}$. We use a linear weighting update rule: $\hat \alpha_{n+1} = t \hat \lambda_n \hat \alpha_n + (1 - t)\hat \alpha_n$ for $t \in (0, 1)$, where $\hat \lambda_n \hat \alpha_n$ is a conservative estimate of $\beta$.

\section{Empirical Evaluation}
We study the empirical properties of DAgger, Imitation Gradient (IG) and Multiple Imitation Gradient (MIG), showing that even in low dimensional and convex settings, the implications of  the convergence of policies become apparent. We intentionally sought out cases and chose parameters such that these on-policy algorithms do not achieve convergence in order to better understand their properties. We consider deterministic domains for the sake of accurately measuring the true dynamic regret. In this evaluation, we attempt to address the following questions: (1) How are policy performance and dynamic regret affected by changing system parameters? (2) Can Adaptive On-Policy Regularization improve convergence of the average dynamic regret and policy performance? The code for all experiments can be found at \texttt{https://github.com/jon--lee/aor}.

\subsection{Cart-Pole Balancing}

First we consider a task where the robot learns to push a cart left or right with a fixed force magnitude in order to balance a pole upright over 100 iterations. The control space is discrete $\left\{\text{left}, \text{right}\right\}$ and the state space consists of cart location and velocity and pole angle and angular velocity. We measure the absolute performance of a policy as the angular deviation from the upright position. We obtained a nonlinear algorithmic supervisor via reinforcement learning. The robot's policy was learned using $l_2$-regularized linear regression as in (\ref{ridge-regression}). In this setting, we vary the difficulty of the problem, i.e. controlling $\beta$, by setting the force magnitude to either low or high values corresponding to easy and hard settings, respectively. For all algorithms, the regularization was initially set to $\hat \alpha_1 = 0.1$. Stepsizes $\eta$ for IG and MIG were set to $0.0001$ and $0.01$, respectively. %Further details can be found in the appendix.

\begin{figure}
	\centering
	\includegraphics[width=4in]{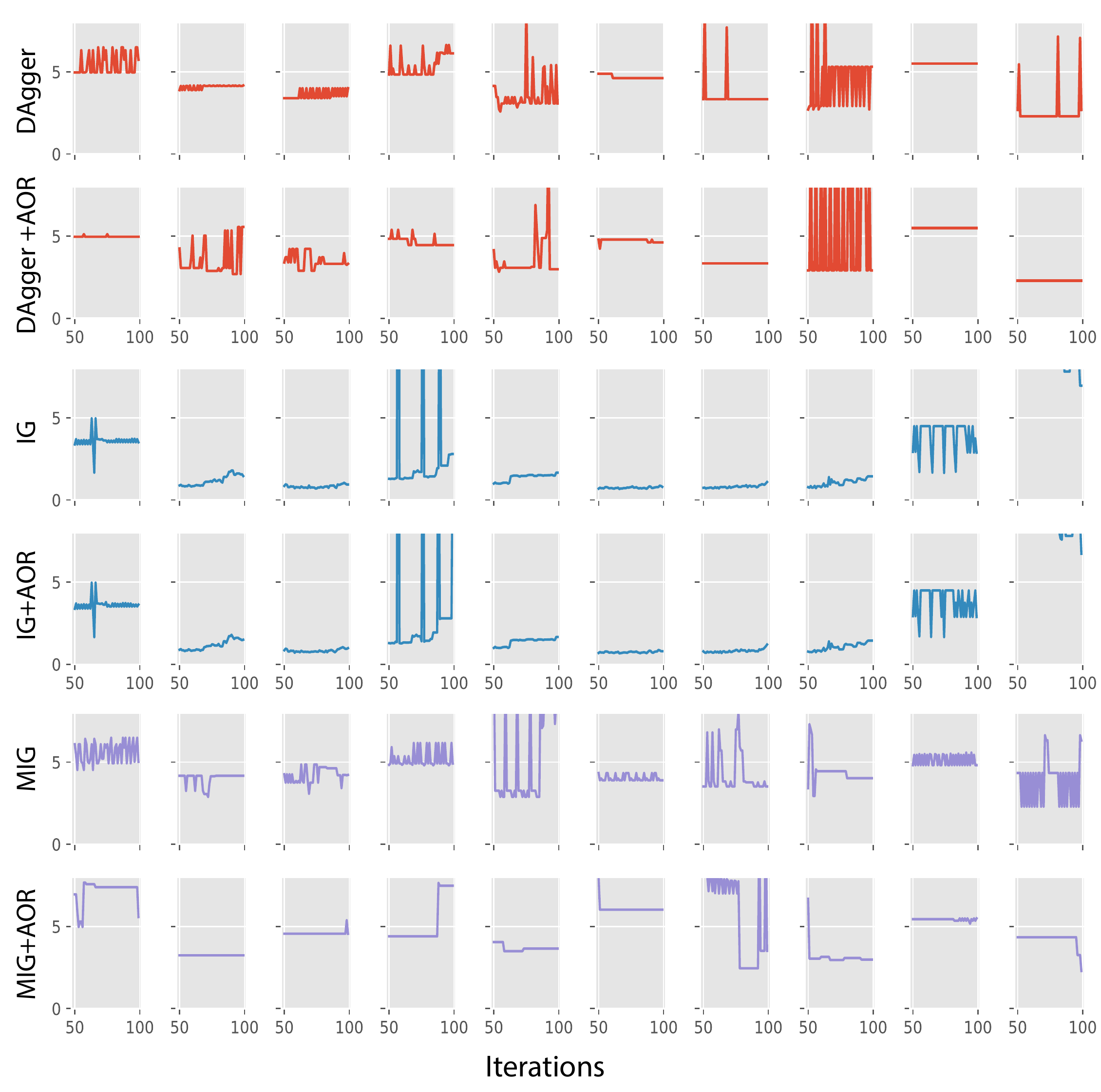}	
	\caption{The cart-pole experiment was repeated over 10 different initial pole angle conditions for each of the algorithms with and without \textsc{Aor}. The cost curves, measured in angular deviation, are shown. Columns correspond to different trials with different initial conditions. Each color group represents a different base algorithm. The top row of each color group is the algorithm without adaptive regularization. The bottom rows utilize adaptive regularization. \textsc{Aor} is generally able to stabilize the cost curves while maintaining low cost.}
	\label{fig:trials}
\end{figure}

To generate the easy and hard versions of cart-pole balancing, we varied the parameter controlling the force magnitude applied with each left or right control. A low force magnitude of $2.0$ was used for the easy setting and a higher force magnitude of $10.0$ was used for the hard setting. The value $2.0$ was the smallest integer before the force was too low to control the cart. The value $10.0$ was one of the highest before we noticed the average $\lambda$ values began to decrease as a function of the force magnitude. As mentioned, the parameters $\eta$ and $\alpha_1$ where intentionally chosen so that the task would exhibit unstable or suboptimal results. Trajectories were 200 time steps long.

In order to estimate $\hat \lambda_n$ between each iteration, we compute $\|\theta_{n} - \theta_{n-1}\|$ directly since both quantities are known. Because each $f_n$ are strongly convex supervised learning problems, full information is known $f_n$ and so the minimum $\theta^*_n = \argmin f_n(\theta)$ can be solved. In this case, $f_n$ corresponded to a $l_2$-regularized ridge regression problem which has a closed form solution. In the case of stochastic problems, it may be necessary to obtain a sample estimate of $f_n$ first by collecting several trajectories per iteration and then estimating $\theta_n^*$ from the sample average. Our estimate of $\hat \lambda_n$ was simply the ratio of these normed differences averaged over iterations. Note that $\lambda$ as defined in Assumption \ref{regularity}, is a global constant, but in practice only local regions may be relevant for the problem at hand, which is why we estimate $\lambda$ at each iteration.

The instantaneous regrets for all three imitation learning algorithms, measured as $f_n(\theta_n) - \min_\theta f_n(\theta)$, are shown in the left column of the top and middle rows of Figure \ref{fig:cartpole} for both the easy and hard settings, respectively. In the right column are the actual angular deviations of the pole. 
%Because the convergence of static regret is guaranteed regardless of the difficulty, we omit it for clarity. 
In the easy setting, the algorithms converge to costs similar to the supervisor. In the hard setting, there is chattering during the learning process. We observe that regret does not always converge indicating a discrepancy between the supervisor and the robot, which is consistent with the theoretical results that suggest difficulty influences convergence to the best policy.

The bottom row of Figure \ref{fig:cartpole} shows that Adaptive On-Policy Regularization (\textsc{Aor}) can be used to improve dynamic regret. In our implementation of \textsc{Aor}, we set $t = 0.01$ and updated $\hat \lambda_n$ and $\hat \alpha_n$ every ten iterations. For IG and MIG with \textsc{Aor}, we adjusted the stepsize $\eta$ as a function of $\hat \alpha_n$ as motivated by the conditions in the theorems. In practice this counteracts potentially large gradients caused by large $\hat \alpha$ values. DAgger without \textsc{Aor} exhibits severe chattering. With \textsc{Aor}, the performance is stabilized, leading to a converged policy. A similar result is observed for MIG. We note that IG did not initially exhibit chattering and the learning curve was unaffected. We attribute this to the discrete nature of the control space; even if the policy parameters have different regret rates, the resulting trajectories could be the same. %We also evaluated the same task over several different initial conditions. Since averaging the results tends to hide the chattering, individual plots are give in the appendix.

We repeated the cart-pole experiment over 10 different initial pole angle conditions in the hard setting in Figure \ref{fig:trials}. The cost curves are shown, measured as angular deviation from the upright position. The top row of each pair shows without \textsc{Aor} and the bottom row shows with \textsc{Aor}. Each column corresponds to a different set of initial conditions. To conserve space, only the last 50 iterations are shown, which is when the curves are typically stabilized by adaptive regularization. DAgger and MIG see a reduction of chattering in most cases when using \textsc{Aor}, while we observe no difference for IG. Interestingly, IG is fairly stable by itself.

\subsection{Walker Locomotion}

\begin{figure}
	\centering
	\includegraphics[width=4in]{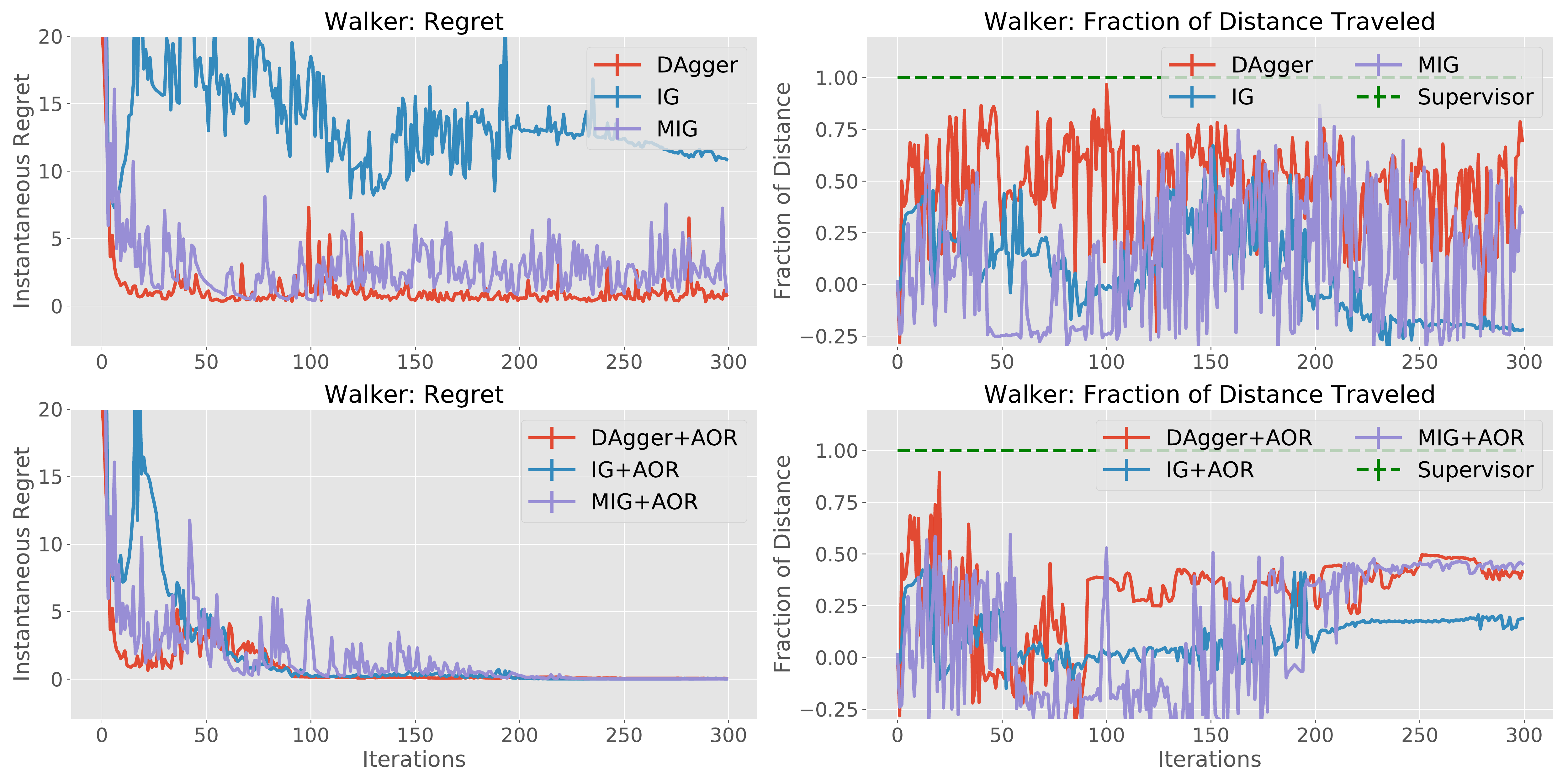}
	\caption{In the 2-dimensional Walker locomotion domain, the instantaneous regrets fail to converge to zero without \textsc{Aor} (top row) and the distance curves exhibit chattering. With \textsc{Aor} (bottom row), the average dynamic regret converges and chattering is reduced after 200 iterations.}
	\label{fig:walker}
	\vspace{-1.5em}
\end{figure}

Next, we consider a 2-dimensional walker from the OpenAI Gym, where the objective is to move the farthest distance. Again we induced difficulty and suboptimal policies by increasing the force magnitude of controls. 
%Additional information on the environment is in the appendix. 
Regularized regression was again used for the robot's policy. Here, we set $\alpha_1 = 1.0$ and $t = 0.1$ for \textsc{Aor}. We used the same initial $\eta$ values for IG and MIG and adjusted the stepsize as a function of $\hat \alpha_n$ when using \textsc{Aor}. The results are shown in Figure \ref{fig:walker}. Without adaptive regularization, the average dynamic regret fails to converge and all distance curves exhibit severe chattering with no stable learning. With \textsc{Aor}, average dynamic regret converges to zero and all distance curves are stabilized.

To induce high regret policies in the walker domain, we simply increased the force of controls five-fold. As in cart-pole, trajectories consisted of 200 time steps and one trajectory was collected and evaluated at each of the 300 iterations. Again, there was no stochasticity in the environment for the sake of computing the instantaneous regret.

\subsection{Hopper Locomotion}

\begin{figure}
	\centering
	\includegraphics[width=4in]{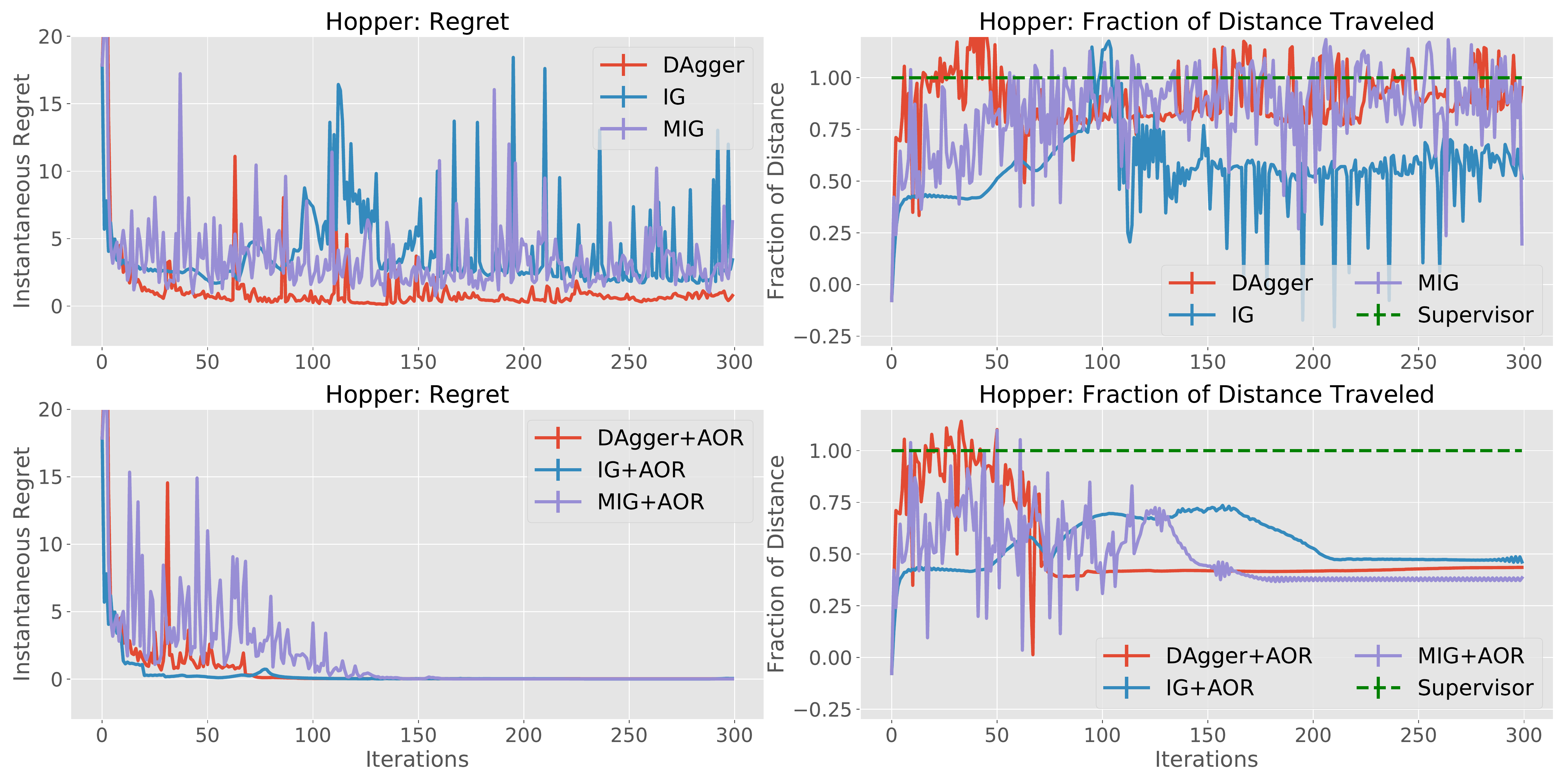}
	\caption{The algorithms are evaluated on OpenAI Gym hopper on the same conditions as the walker experiment. As observed in prior experiments, on-policy algorithms without proper regularization lead to unstable learning. With adaptive regularization the learning is stabilized; however, in this case, performance on the system is reduced on average.}
	\label{fig:hopper}
\end{figure}

We also evaluated the effect of \textsc{Aor} on a different OpenAI gym task, hopper, with exactly the same hyperparameter settings. The results are shown in Figure \ref{fig:hopper}. We note that although chattering is reduced, the distance traveled on average suffers. This is one disadvantage of adaptive regularization that was noted earlier. Increased regularization, although leading to convergence of average dynamic regret, can cause poor policy performance if it is excessive. This example reaffirms that from a practical perspective in imitation learning, it is important to monitor both the surrogate loss and whatever qualitative or quantitative metrics are actually desired to ensure that they agree.

\section{Discussion and Future Work}

Dynamic regret theory offers a promising framework for theoretical analysis in imitation learning. The question of whether an on-policy algorithm leads to converged or stable policies is inherently captured in dynamic regret, in contrast to static regret which captures policy performance on the average of the distributions. The theoretical analyses suggest new conditions to guarantee convergence. The simulation results suggest that the questions of convergence and optimality must be addressed when designing imitation learning robotics systems because stable performance is not always guaranteed. 
Indeed, even if static regret is found to be low, the distributions induced by the robot during training may be far too unpredictable for stable policy performance to actually be achieved.

Our results suggest that if the problem is sufficiently well-conditioned, we can guarantee convergence to a unique solution for these on-policy algorithms. This is surprising because intuitively initialization would seem to be an important factor, but does not ultimately affect the solution in the limit. We also draw connections between the on-policy imitation learning problem and variational inequalities and fixed-point problems by virtue of the continuous online learning framework \citep{cheng2019online}. Further exploration of these connections may reveal additional solution characteristics and efficient algorithms for solving imitation learning problems.

By modeling the problem with Assumption \ref{regularity}, we constrained the dynamics to aid the analysis. If other properties were known about the dynamics, then additional information could improve regret rates and conditions for optimality. For example, similar problem statements were studied in \cite{hall2015online} and  \cite{rakhlin2013optimization} in general online optimization. It was recently shown that augmenting the mostly model-free analyses with model-based learning can improve static regret bounds and performance \citep{cheng2019accelerating}. We hypothesize that dynamic regret rates can also be improved.

Furthermore, Assumption~\ref{regularity} excludes certain difficult systems with discontinuous dynamics such as hybrid systems. While the purpose of this article is to provide foundational theoretical results in the continuous setting, it is still not known whether these results can be applied to systems where the loss may be discontinuous on the parameter space. One possible direction is to show there are local regions that are sufficiently continuous to apply the current results and also invariant under the given on-policy algorithm. The problem then reduces to reaching these invariant sets. We leave exploration of this fascinating direction of discontinuities for future work.

\bibliographystyle{plainnat}
\bibliography{main} 

\begin{thebibliography}{48}
\providecommand{\natexlab}[1]{#1}
\providecommand{\url}[1]{\texttt{#1}}
\expandafter\ifx\csname urlstyle\endcsname\relax
  \providecommand{\doi}[1]{doi: #1}\else
  \providecommand{\doi}{doi: \begingroup \urlstyle{rm}\Url}\fi

\bibitem[Adamskiy et~al.(2012)Adamskiy, Koolen, Chernov, and
  Vovk]{adamskiy2012closer}
Dmitry Adamskiy, Wouter~M Koolen, Alexey Chernov, and Vladimir Vovk.
\newblock A closer look at adaptive regret.
\newblock In \emph{International Conference on Algorithmic Learning Theory},
  pages 290--304. Springer, 2012.

\bibitem[Asadi et~al.(2018)Asadi, Misra, and Littman]{asadi2018lipschitz}
Kavosh Asadi, Dipendra Misra, and Michael Littman.
\newblock Lipschitz continuity in model-based reinforcement learning.
\newblock In \emph{International Conference on Machine Learning}, pages
  264--273, 2018.

\bibitem[Bagnell(2015)]{bagnell2015invitation}
J~Andrew Bagnell.
\newblock An invitation to imitation.
\newblock Technical report, Carnegie Mellon Univ Pittsburgh PA Robotics Inst,
  2015.

\bibitem[Banach(1922)]{banach1922operations}
Stefan Banach.
\newblock Sur les op{\'e}rations dans les ensembles abstraits et leur
  application aux {\'e}quations int{\'e}grales.
\newblock \emph{Fund. math}, 3\penalty0 (1):\penalty0 133--181, 1922.

\bibitem[Bertsekas(1975)]{bertsekas1975convergence}
D~Bertsekas.
\newblock Convergence of discretization procedures in dynamic programming.
\newblock \emph{IEEE Transactions on Automatic Control}, 20\penalty0
  (3):\penalty0 415--419, 1975.

\bibitem[Bubeck(2011)]{bubeck2011introduction}
S{\'e}bastien Bubeck.
\newblock Introduction to online optimization.
\newblock \emph{Lecture Notes}, pages 1--86, 2011.

\bibitem[Cheng and Boots(2018)]{cheng2018convergence}
Ching-An Cheng and Byron Boots.
\newblock Convergence of value aggregation for imitation learning.
\newblock \emph{International Conference on Artificial Intelligence and
  Statistics}, 2018.

\bibitem[Cheng et~al.(2018)Cheng, Yan, Wagener, and Boots.]{cheng2018fast}
Ching-An Cheng, Xinyan Yan, Nolan Wagener, and Byron Boots.
\newblock Fast policy learning through imitation and reinforcement.
\newblock In \emph{Conference on Uncertainty in Artificial Intelligence}, 2018.

\bibitem[Cheng et~al.(2019{\natexlab{a}})Cheng, Lee, Goldberg, and
  Boots]{cheng2019online}
Ching-An Cheng, Jonathan Lee, Ken Goldberg, and Byron Boots.
\newblock Online learning with continuous variations: Dynamic regret and
  reductions.
\newblock \emph{arXiv preprint arXiv:1902.07286}, 2019{\natexlab{a}}.

\bibitem[Cheng et~al.(2019{\natexlab{b}})Cheng, Yan, Theodorou, and
  Boots]{cheng2019accelerating}
Ching-An Cheng, Xinyan Yan, Evangelos Theodorou, and Byron Boots.
\newblock Accelerating imitation learning with predictive models.
\newblock In \emph{International Conference on Artificial Intelligence and
  Statistics (AISTATS)}, 2019{\natexlab{b}}.

\bibitem[Duan et~al.(2017)Duan, Andrychowicz, Stadie, Ho, Schneider, Sutskever,
  Abbeel, and Zaremba]{duan2017one}
Yan Duan, Marcin Andrychowicz, Bradly Stadie, OpenAI~Jonathan Ho, Jonas
  Schneider, Ilya Sutskever, Pieter Abbeel, and Wojciech Zaremba.
\newblock One-shot imitation learning.
\newblock In \emph{Advances in neural information processing systems}, pages
  1087--1098, 2017.

\bibitem[Duchi et~al.(2011)Duchi, Hazan, and Singer]{duchi2011adaptive}
John Duchi, Elad Hazan, and Yoram Singer.
\newblock Adaptive subgradient methods for online learning and stochastic
  optimization.
\newblock \emph{Journal of Machine Learning Research}, 12\penalty0
  (Jul):\penalty0 2121--2159, 2011.

\bibitem[Duvallet et~al.(2013)Duvallet, Kollar, and
  Stentz]{duvallet2013imitation}
Felix Duvallet, Thomas Kollar, and Anthony Stentz.
\newblock Imitation learning for natural language direction following through
  unknown environments.
\newblock In \emph{ICRA}, pages 1047--1053. IEEE, 2013.

\bibitem[Facchinei and Pang(2007)]{facchinei2007finite}
Francisco Facchinei and Jong-Shi Pang.
\newblock \emph{Finite-dimensional variational inequalities and complementarity
  problems}.
\newblock Springer Science \& Business Media, 2007.

\bibitem[Fukushima(1996)]{fukushima1996merit}
Masao Fukushima.
\newblock Merit functions for variational inequality and complementarity
  problems.
\newblock In \emph{Nonlinear Optimization and Applications}, pages 155--170.
  Springer, 1996.

\bibitem[Hall and Willett(2015)]{hall2015online}
Eric~C Hall and Rebecca~M Willett.
\newblock Online convex optimization in dynamic environments.
\newblock \emph{IEEE Journal of Selected Topics in Signal Processing},
  9\penalty0 (4):\penalty0 647--662, 2015.

\bibitem[Hazan(2016)]{hazan2016introduction}
Elad Hazan.
\newblock Introduction to online convex optimization.
\newblock \emph{Foundations and Trends in Optimization}, 2\penalty0
  (3-4):\penalty0 157--325, 2016.

\bibitem[Hazan and Kale(2014)]{hazan2014beyond}
Elad Hazan and Satyen Kale.
\newblock Beyond the regret minimization barrier: optimal algorithms for
  stochastic strongly-convex optimization.
\newblock \emph{The Journal of Machine Learning Research}, 15\penalty0
  (1):\penalty0 2489--2512, 2014.

\bibitem[Hazan and Seshadhri(2007)]{hazan2007adaptive}
Elad Hazan and Comandur Seshadhri.
\newblock Adaptive algorithms for online decision problems.
\newblock \emph{Electronic colloquium on computational complexity (ECCC)},
  2007.

\bibitem[Hazan et~al.(2007)Hazan, Agarwal, and Kale]{hazan2007logarithmic}
Elad Hazan, Amit Agarwal, and Satyen Kale.
\newblock Logarithmic regret algorithms for online convex optimization.
\newblock \emph{Machine Learning}, 69\penalty0 (2-3):\penalty0 169--192, 2007.

\bibitem[Hinderer(2005)]{hinderer2005lipschitz}
Karl Hinderer.
\newblock Lipschitz continuity of value functions in markovian decision
  processes.
\newblock \emph{Mathematical Methods of Operations Research}, 62\penalty0
  (1):\penalty0 3--22, 2005.

\bibitem[Hussein et~al.(2018)Hussein, Elyan, Gaber, and Jayne]{hussein2018deep}
Ahmed Hussein, Eyad Elyan, Mohamed~Medhat Gaber, and Chrisina Jayne.
\newblock Deep imitation learning for 3d navigation tasks.
\newblock \emph{Neural Computing and Applications}, 29\penalty0 (7):\penalty0
  389--404, 2018.

\bibitem[Jadbabaie et~al.(2015)Jadbabaie, Rakhlin, Shahrampour, and
  Sridharan]{jadbabaie2015online}
Ali Jadbabaie, Alexander Rakhlin, Shahin Shahrampour, and Karthik Sridharan.
\newblock Online optimization: Competing with dynamic comparators.
\newblock In \emph{Artificial Intelligence and Statistics}, pages 398--406,
  2015.

\bibitem[Ke et~al.(2019)Ke, Barnes, Sun, Lee, Choudhury, and
  Srinivasa]{ke2019imitation}
Liyiming Ke, Matt Barnes, Wen Sun, Gilwoo Lee, Sanjiban Choudhury, and
  Siddhartha Srinivasa.
\newblock Imitation learning as $ f $-divergence minimization.
\newblock \emph{arXiv preprint arXiv:1905.12888}, 2019.

\bibitem[Khalil and Grizzle(2002)]{khalil2002nonlinear}
Hassan~K Khalil and Jessy~W Grizzle.
\newblock \emph{Nonlinear systems}, volume~3.
\newblock Prentice hall Upper Saddle River, NJ, 2002.

\bibitem[Laskey et~al.(2017)Laskey, Chuck, Lee, Mahler, Krishnan, Jamieson,
  Dragan, and Goldberg]{laskey2017comparing}
Michael Laskey, Caleb Chuck, Jonathan Lee, Jeffrey Mahler, Sanjay Krishnan,
  Kevin Jamieson, Anca Dragan, and Ken Goldberg.
\newblock Comparing human-centric and robot-centric sampling for robot deep
  learning from demonstrations.
\newblock In \emph{IEEE International Conference on Robotics and Automation
  (ICRA), 2017}, 2017.

\bibitem[Laskey(2018)]{laskey2018and}
Michael~David Laskey.
\newblock \emph{On and Off-Policy Deep Imitation Learning for Robotics}.
\newblock PhD thesis, UC Berkeley, 2018.

\bibitem[Le et~al.(2016)Le, Kang, Yue, and Carr]{le2016smooth}
Hoang Le, Andrew Kang, Yisong Yue, and Peter Carr.
\newblock Smooth imitation learning for online sequence prediction.
\newblock In \emph{International Conference on Machine Learning}, pages
  680--688, 2016.

\bibitem[Lee et~al.(2018{\natexlab{a}})Lee, Laskey, Tanwani, Aswani, and
  Goldberg]{lee2018dynamic}
Jonathan Lee, Michael Laskey, Ajay~Kumar Tanwani, Anil Aswani, and Ken
  Goldberg.
\newblock A dynamic regret analysis and adaptive regularization algorithm for
  on-policy robot imitation learning.
\newblock \emph{Workshop on Algorithmic Foundations of Robotics (WAFR)},
  2018{\natexlab{a}}.

\bibitem[Lee et~al.(2018{\natexlab{b}})Lee, Laskey, Tanwani, and
  Goldberg]{lee2018stability}
Jonathan Lee, Michael Laskey, Ajay~Kumar Tanwani, and Ken Goldberg.
\newblock Stability analysis of on-policy imitation learning algorithms using
  dynamic regret.
\newblock In \emph{RSS Workshop on Imitation and Causality},
  2018{\natexlab{b}}.

\bibitem[Mokhtari et~al.(2016)Mokhtari, Shahrampour, Jadbabaie, and
  Ribeiro]{mokhtari2016online}
Aryan Mokhtari, Shahin Shahrampour, Ali Jadbabaie, and Alejandro Ribeiro.
\newblock Online optimization in dynamic environments: Improved regret rates
  for strongly convex problems.
\newblock In \emph{Decision and Control (CDC), 2016 IEEE 55th Conference on},
  pages 7195--7201. IEEE, 2016.

\bibitem[Osa et~al.(2018)Osa, Pajarinen, Neumann, Bagnell, Abbeel, Peters,
  et~al.]{osa2018algorithmic}
Takayuki Osa, Joni Pajarinen, Gerhard Neumann, J~Andrew Bagnell, Pieter Abbeel,
  Jan Peters, et~al.
\newblock An algorithmic perspective on imitation learning.
\newblock \emph{Foundations and Trends in Robotics}, 7\penalty0 (1-2):\penalty0
  1--179, 2018.

\bibitem[Pan et~al.(2018)Pan, Cheng, Saigol, Lee, Yan, Theodorou, and
  Boots]{pan2017agile}
Yunpeng Pan, Ching-An Cheng, Kamil Saigol, Keuntaek Lee, Xinyan Yan, Evangelos
  Theodorou, and Byron Boots.
\newblock Agile off-road autonomous driving using end-to-end deep imitation
  learning.
\newblock In \emph{Robotics: Science and Systems}, 2018.

\bibitem[Pirotta et~al.(2015)Pirotta, Restelli, and
  Bascetta]{pirotta2015policy}
Matteo Pirotta, Marcello Restelli, and Luca Bascetta.
\newblock Policy gradient in lipschitz markov decision processes.
\newblock \emph{Machine Learning}, 100\penalty0 (2-3):\penalty0 255--283, 2015.

\bibitem[Pomerleau(1989)]{pomerleau1989alvinn}
Dean~A Pomerleau.
\newblock Alvinn: An autonomous land vehicle in a neural network.
\newblock Technical report, Carnegie-Mellon University, 1989.

\bibitem[Rakhlin and Sridharan(2014)]{rakhlin2014statistical}
A~Rakhlin and K~Sridharan.
\newblock Statistical learning and sequential prediction.
\newblock \emph{Book Draft}, 2014.

\bibitem[Rakhlin and Sridharan(2013)]{rakhlin2013optimization}
Sasha Rakhlin and Karthik Sridharan.
\newblock Optimization, learning, and games with predictable sequences.
\newblock In \emph{Advances in Neural Information Processing Systems}, pages
  3066--3074, 2013.

\bibitem[Ross et~al.(2011)Ross, Gordon, and Bagnell]{ross2010reduction}
St{\'e}phane Ross, Geoffrey~J Gordon, and J~Andrew Bagnell.
\newblock A reduction of imitation learning and structured prediction to
  no-regret online learning.
\newblock \emph{International Conference on Artificial Intelligence and
  Statistics}, 2011.

\bibitem[Ross et~al.(2013)Ross, Melik-Barkhudarov, Shankar, Wendel, Dey,
  Bagnell, and Hebert]{ross2013learning}
St{\'e}phane Ross, Narek Melik-Barkhudarov, Kumar~Shaurya Shankar, Andreas
  Wendel, Debadeepta Dey, J~Andrew Bagnell, and Martial Hebert.
\newblock Learning monocular reactive uav control in cluttered natural
  environments.
\newblock In \emph{2013 IEEE international conference on robotics and
  automation}, pages 1765--1772. IEEE, 2013.

\bibitem[Sastry(1999)]{sastry2013nonlinear}
Shankar Sastry.
\newblock \emph{Nonlinear systems: analysis, stability, and control},
  volume~10.
\newblock Springer-Verlag New York, 1999.

\bibitem[Schulman et~al.(2015)Schulman, Levine, Abbeel, Jordan, and
  Moritz]{schulman2015trust}
John Schulman, Sergey Levine, Pieter Abbeel, Michael Jordan, and Philipp
  Moritz.
\newblock Trust region policy optimization.
\newblock In \emph{International Conference on Machine Learning (ICML)}, 2015.

\bibitem[Shalev-Shwartz and Kakade(2009)]{shalev2009mind}
Shai Shalev-Shwartz and Sham~M Kakade.
\newblock Mind the duality gap: Logarithmic regret algorithms for online
  optimization.
\newblock In \emph{Advances in Neural Information Processing Systems}, pages
  1457--1464, 2009.

\bibitem[Sun et~al.(2017)Sun, Venkatraman, Gordon, Boots, and
  Bagnell]{sun2017deeply}
Wen Sun, Arun Venkatraman, Geoffrey~J Gordon, Byron Boots, and J~Andrew
  Bagnell.
\newblock Deeply aggrevated: Differentiable imitation learning for sequential
  prediction.
\newblock In \emph{International Conference on Machine Learning}, 2017.

\bibitem[Yang et~al.(2016)Yang, Zhang, Jin, and Yi]{yang2016tracking}
Tianbao Yang, Lijun Zhang, Rong Jin, and Jinfeng Yi.
\newblock Tracking slowly moving clairvoyant: Optimal dynamic regret of online
  learning with true and noisy gradient.
\newblock In \emph{International Conference on Machine Learning}, 2016.

\bibitem[Zhang and Cho(2017)]{zhang2017query}
Jiakai Zhang and Kyunghyun Cho.
\newblock Query-efficient imitation learning for end-to-end simulated driving.
\newblock In \emph{Thirty-First AAAI Conference on Artificial Intelligence},
  2017.

\bibitem[Zhang et~al.(2017)Zhang, Yang, Yi, Rong, and Zhou]{zhang2017improved}
Lijun Zhang, Tianbao Yang, Jinfeng Yi, Jing Rong, and Zhi-Hua Zhou.
\newblock Improved dynamic regret for non-degenerate functions.
\newblock In \emph{Advances in Neural Information Processing Systems}, 2017.

\bibitem[Zhang et~al.(2018)Zhang, McCarthy, Jowl, Lee, Chen, Goldberg, and
  Abbeel]{zhang2018deep}
Tianhao Zhang, Zoe McCarthy, Owen Jowl, Dennis Lee, Xi~Chen, Ken Goldberg, and
  Pieter Abbeel.
\newblock Deep imitation learning for complex manipulation tasks from virtual
  reality teleoperation.
\newblock In \emph{2018 IEEE International Conference on Robotics and
  Automation (ICRA)}, pages 1--8. IEEE, 2018.

\bibitem[Zinkevich(2003)]{zinkevich2003online}
Martin Zinkevich.
\newblock Online convex programming and generalized infinitesimal gradient
  ascent.
\newblock In \emph{Proceedings of the 20th International Conference on Machine
  Learning (ICML-03)}, pages 928--936, 2003.

\end{thebibliography}

\appendix

\onecolumn

\section{Omitted Proofs}

% \red{\sout{For convenience, let $\nabla f_n(\theta) := \nabla_\theta f_n(\theta)$ denote the gradient in the evaluation parameter at the $n$th iteration.}} 

\normalsize{For convenience, the main assumptions of the paper are reproduced here:}
\begin{itemize}
	\item For all $\theta_1, \theta_2, \theta \in \Theta$, $\exists \alpha > 0$ such that
	\begin{align*}
	f_\theta(\theta_2) \geq f_\theta(\theta_1) + \langle \nabla f_\theta( \theta_1), \theta_2 - \theta_1 \rangle + \frac{\alpha}{2}\|\theta_1 - \theta_2\|^2.
	\end{align*}
	\item For all $\theta_1, \theta_2, \theta \in \Theta$, $\exists \gamma > 0$ such that
	\begin{align*}
	\|\nabla f_\theta (\theta_1) - \nabla f_\theta ( \theta_2) \| \leq \gamma \| \theta_1 - \theta_2\|
	\end{align*}
	and $\exists G > 0$ such that $\|\nabla f_\theta (\theta_1)\| \leq G$.
	\item For all $\theta' \in \Theta$, $\theta^*$ is in the relative interior of $\Theta$ where $\theta^* = \argmin_{\theta \in \Theta} f_{\theta'}(\theta)$. That is, $\nabla f_{\theta'}(\theta^*) = 0$.
	\item For all $\theta_1, \theta_2, \theta \in \Theta$, $\exists \beta> 0$ such that
	\begin{align*}
	\|\nabla f_{\theta_1}(\theta) - \nabla f_{\theta_2}(\theta)\| \leq \beta \|\theta_1 - \theta_2\|.
	\end{align*}
\end{itemize}

\subsection{Proof of Lemma \ref{lemma:decreasing-distance}}
The result of Lemma~\ref{lemma:decreasing-distance} is based on a technical lemma from \cite{zhang2017improved}.
%\begin{lemma}\label{lemma:decreasing-distance}
%	Let $\theta'$ be the current parameter played by the algorithm at any iteration, $\hat \theta= P_\Theta(\theta' - \eta \nabla f_n(\theta'))$ and $\theta_n^* = \argmin_\theta f_n(\theta)$. Then we have
%	\begin{align*}
%	\|\hat \theta - \theta_n^*\|^2 \leq \left(1 - \frac{2\alpha}{1/\eta + \alpha } \right) \|\theta' - \theta_n^* \|^2.
%	\end{align*}
%\end{lemma}
\begin{proof}
	By the update rule:
	\begin{align*}
	\hat \theta & = P_\Theta(\theta' - \eta \nabla f_n(\theta')) \\
	& = \argmin_{\theta \in \Theta} \|\theta' - \eta \nabla f_n(\theta') - \theta \|^2 \\
	& = \argmin_{\theta \in \Theta} \left\{2 \langle \eta\nabla f_n(\theta'), \theta  - \theta' \rangle + \|\theta' - \theta\|^2 + \|\eta^2 \nabla f_n(\theta') \|^2\right\} \\
	& = \argmin_{\theta \in \Theta} \left\{ f_n(\theta') + \langle \nabla f_n(\theta'), \theta  - \theta' \rangle + \frac{1}{2\eta}\|\theta' - \theta\|^2 \right\}  \\
	& = \argmin_{\theta \in \Theta} h_n(\theta)
	\end{align*}
	where we define $h_n(\theta) := f_n(\theta') + \langle \nabla f_n(\theta'), \theta  - \theta' \rangle + \frac{1}{2\eta}\|\theta' - \theta\|^2$. Note that $h(\theta)$ is $\frac{1}{\eta}$-strongly convex. So by applying Lemma \ref{lemma:stronglyconvex-simple} to $h_n$ and by the fact that $\hat \theta$ is the minimizer of $h$, we have 
	\begin{align*}
	h_n(\hat \theta) & \leq h_n(\theta_n^*) - \frac{1}{2\eta} \|\hat \theta - \theta^*_n\|^2
	\end{align*}
	By the strong convexity of $f_n$ it holds that $f_n(\theta') + \langle \nabla f_n(\theta'), \theta^*_n - \theta' \rangle \leq f_n(\theta^*_n) - \frac{\alpha}{2}\|\theta' - \theta^*_n\|^2$. By smoothness and the fact that $\eta < 1/\gamma$, we also have
	\begin{align*}
	f_n(\hat \theta) \leq f_n(\theta') + \langle \nabla f_n(\theta'), \hat \theta  - \theta' \rangle + \frac{\gamma}{2}\|\theta' - \hat \theta\|^2 \leq h_n(\hat \theta)
	\end{align*}
	Combining these inequalities gives
	\begin{align*}
	f_n(\hat \theta) \leq f_n(\theta^*_n) - \frac{\alpha}{2}\|\theta' - \theta_n^*\|^2 + \frac{1}{2\eta}\|\theta' - \theta_{n}^*\|^2 - \frac{1}{2\eta} \| \hat \theta - \theta_n^* \|^2
	\end{align*}
	By applying Lemma \ref{lemma:stronglyconvex-simple} again we have:
	\begin{align*}
	\frac{\alpha}{2}\|\hat \theta - \theta_n^*\|^2 \leq - \frac{\alpha}{2}\|\theta' - \theta_{n}^*\|^2 + \frac{1}{2\eta}\|\theta' - \theta_{n}^*\|^2 - \frac{1}{2\eta} \| \hat \theta - \theta_n^* \|^2
	\end{align*}
	The result can be obtained by rearranging and aggregating the terms and then simplifying.
	\end{proof}

\subsection{Proof of Lemma \ref{zhang-theorem}}

The result of Lemma \ref{zhang-theorem} is a slight modification from Theorem 3 from \cite{zhang2017improved}, the proof of which is reproduced in full here for completeness. Again before proving this theorem, we establish a crucial lemma. As in the proof of Theorem \ref{ogd_theorem}, we show a bound on the improvement from a single gradient step.

\begin{proof}%[Proof of Lemma \ref{zhang-theorem}]
	According to the update rule of multiple gradients we apply the result established in Lemma \ref{lemma:decreasing-distance} $K$ times which gives
	\begin{align*}
	\|\theta_{n+1} - \theta_n^*\|^2 & \leq \left(1 - \frac{2\alpha}{1/\eta + \alpha} \right)^K \|\theta_{n} - \theta_n^*\|^2 \\
	& \leq  \exp\left(- \frac{2\alpha K}{1/\eta + \alpha} \right) \|\theta_{n} - \theta_n^*\|^2 \\
	& \leq \frac{1}{4}\|\theta_n - \theta_n^*\|^2
	\end{align*}
	Then, as in Theorem \ref{ogd_theorem}, we bound the distances between the optimal parameters and the algorithm's parameters.
	\begin{align*}
	\sum_{n=1}^N\| \theta_n - \theta_n^*\|^2
	& = \| \theta_1 - \theta_1^*\|^2 + \sum_{n=2}^N\| \theta_n - \theta_{n-1}^* + \theta_{n-1}^* -  \theta_n^*\|^2 \\
	& \leq \| \theta_1 - \theta_1^*\|^2 + 2\sum_{n=2}^N \| \theta_n - \theta_{n-1}^*\|^2 + \|\theta_{n-1}^* -  \theta_n^*\|^2 \\
	& \leq \| \theta_1 - \theta_1^*\|^2 + \sum_{n=2}^N \frac{1}{2} \| \theta_{n-1} - \theta_{n-1}^*\|^2 + 2\|\theta_{n-1}^* -  \theta_n^*\|^2 \\
	& \leq 2\| \theta_1 - \theta_1^*\|^2 + 4\sum_{n=2}^N \|\theta_{n-1}^* -  \theta_n^*\|^2 \\
	& \leq 2\| \theta_1 - \theta_1^*\|^2 + 4S\left(\theta^*_{1:N}\right)
	\end{align*}
	Finally, by the smoothness of all $f_n$ we have
	\begin{align*}
	\sum_{n = 1}^N f_n(\theta_n) - f_n(\theta_n^*) & \leq \sum_{n=1}^N\frac{\gamma}{2}\| \theta_n - \theta_n^*\|^2 \\
	& \leq 2\gamma S(\theta^*_{1:N}) + \gamma \|\theta_1 - \theta_1^*\|^2.
	\end{align*}
	\end{proof}

\end{document}